\documentclass[12pt]{colt2024}% Include author names

\usepackage[margin=1in]{geometry}

\usepackage{xcolor}
\colorlet{darkgreen}{green!50!black}
\usepackage{tikz}
\usetikzlibrary{3d}
\usepackage{framed}
\usetikzlibrary{matrix,arrows,snakes,shapes,trees,positioning}

\tikzset{
    >=latex,
    punkt/.style={
           rectangle,
           rounded corners,
           draw=black, very thick,
           text width=6.5em,
           minimum height=2em,
           text centered},
    pil/.style={
           ->,
           double,
           thick,
           shorten <=2pt,
           shorten >=2pt,},
    punkti/.style={
           rectangle,
           rounded corners,
           draw=black, very thick,
           text width=26.5em,
           minimum height=2em,
           text centered},
    punktii/.style={
           rectangle,
           rounded corners,
           draw=black, very thick,
           text width=18.5em,
           minimum height=2em,
           text centered}
}

\usepackage{amsmath,amsfonts,amssymb,mathtools}

\usepackage{cleveref}
\newtheorem{open}{Open Question}
\newtheorem{claim}{Claim}
\newtheorem*{definition*}{Definition}
\newtheorem*{theorem*}{Theorem}
\newtheorem*{open*}{Open Question}

\renewcommand{\H}{\mathcal{H}}
\newcommand{\C}{\mathcal{C}}
\newcommand{\X}{\mathcal{X}}
\newcommand{\F}{\mathcal{F}}

\newcommand{\DS}{\mathtt{DS}}

\newcommand{\case}[4]{ \begin{cases}
    #1\quad &#2\\
    #3\quad &#4\end{cases}}
\newcommand{\Y}{\mathcal{Y}}
\newcommand{\Z}{\mathcal{Z}}
\renewcommand{\H}{\mathcal{H}}
\newcommand{\G}{\mathcal{G}}
\newcommand{\Gk}{\mathtt{G}_k}
\newcommand{\D}{\mathcal{D}}
\newcommand{\A}{\mathcal{A}}
\newcommand{\U}{\mathcal{U}}
\newcommand{\V}{\mathcal{V}}
\newcommand{\dis}{\prec}
\newcommand{\eps}{\varepsilon}
\newcommand{\E}{\mathop{\mathbb{E}}}

\newcommand{\supp}{\text{supp}}

\title[List Sample Compression and Uniform Convergence]{List Sample Compression and Uniform Convergence}
\usepackage{times}
\coltauthor{%
 \Name{Steve Hanneke}\Email{steve.hanneke@gmail.com}\\
 \addr Department of Computer Science, Purdue University
 \AND
 \Name{Shay Moran}\Email{smoran@technion.ac.il}\\
 \addr Departments of Mathematics, Computer Science, Data and Decision Sciences, Technion, and Google Research
 \AND
 \Name{Tom Waknine}\Email{tom.waknine@campus.technion.ac.il}\\
 \addr Department of Mathematics, Technion
}

\begin{document}
\maketitle

\begin{abstract}
List learning is a variant of supervised classification where the learner outputs multiple plausible labels for each instance rather than just one. 
We investigate classical principles related to generalization within the context of list learning. Our primary goal is to determine whether classical principles in the PAC setting retain their applicability in the domain of list PAC learning. We focus on uniform convergence (which is the basis of Empirical Risk Minimization) and on sample compression (which is a powerful manifestation of Occam's Razor). In classical PAC learning, both uniform convergence and sample compression satisfy a form of `completeness': whenever a class is learnable, it can also be learned by a learning rule that adheres to these principles. We ask whether the same completeness holds true in the list learning setting.

We show that uniform convergence remains equivalent to learnability in the list PAC learning setting. In contrast, our findings reveal surprising results regarding sample compression: we prove that when the label space is $\Y=\{0,1,2\}$, then there are 2-list-learnable classes that cannot be compressed. This refutes the list version of the sample compression conjecture by \cite{littlestone:86}. We prove an even stronger impossibility result, showing that there are $2$-list-learnable classes that cannot be compressed even when the reconstructed function can work with lists of arbitrarily large size.
We prove a similar result for (1-list) PAC learnable classes when the label space is unbounded. This generalizes a recent result by \cite{pabbaraju:23}. 

In our impossibility results on sample compression, we employ direct-sum arguments which might be of independent interest. In fact, these arguments raise natural open questions that we leave for future research. Our findings regarding uniform convergence rely on a coding theoretic perspective.

% We show that empirical risk minimization and uniform convergence retain their applicability in the list PAC learning setting. In contrast, our findings reveal surprising results regarding sample compression: we prove that there are 2-list-learnable classes that cannot be compressed, even when the reconstructed function can work with lists of arbitrary size. This refutes the list variant of Littlestone and Warmuth's conjecture~\cite{littlestone:86}.
% We prove a similar result for PAC learnable classes when the label space is unbounded.
% This generalizes a recent result by Pabbaraju~\cite{pabbaraju:23}. 
\end{abstract}

\section{Introduction}
List learning is a natural generalization of supervised classification, in which, instead of predicting the correct label, the learner outputs a small list of labels, one of which should be the correct one. This approach can be viewed as giving the learner more than one guess at the correct label. 

There are many settings in which one may prefer the list learning approach to the classical one. For example, recommendation systems often suggest a short list of products to users, with the hope that the customer will be interested in one of them (see \Cref{fig:list recommendation}).
\begin{figure}
    \centering
    \includegraphics[scale=0.5]{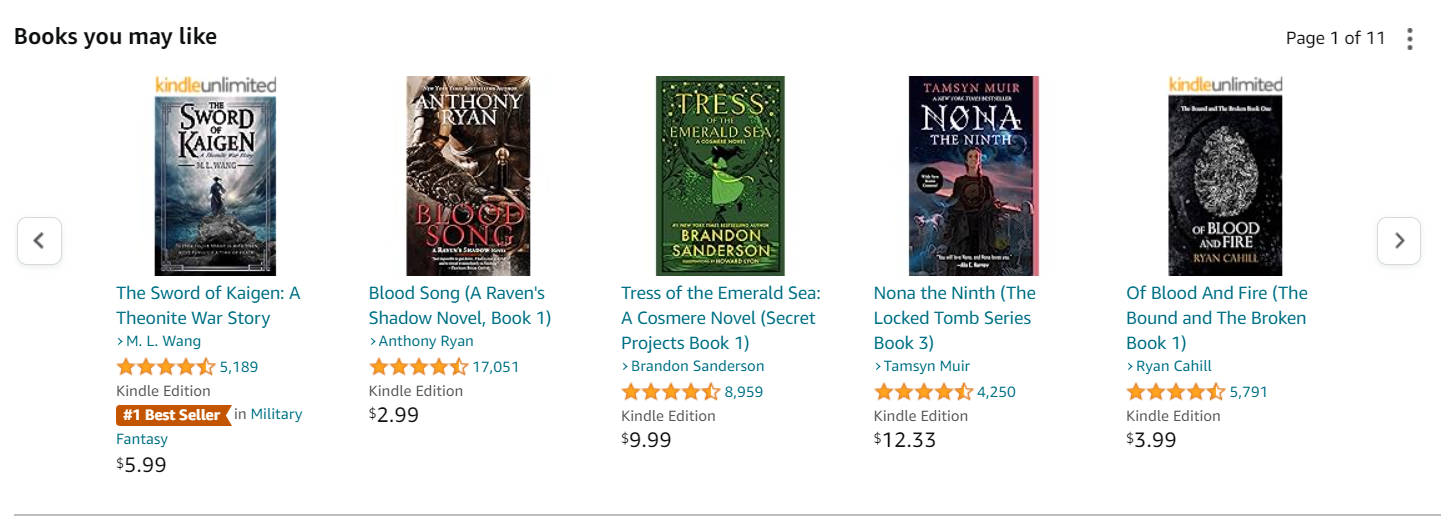}
    \caption{Amazon recommendation system gives their users a short list of books based on their past reading, aiming that one of those books will capture their interest. }
    \label{fig:list recommendation}
\end{figure}
 Another example is the top-$k$ loss function in which the model gets $k$ guesses for each sample. This loss function is often used in ML competitions and can be seen as a variant of list learning. Additionally, list learning addresses label ambiguity; for example, in computer vision recognition problems, it is often impossible to determine if a certain image is of a pond or a river. As a result, training a model for such problems by penalizing it for every mistake can be too restrictive. However, using a top-$k$ approach seems like a reasonable alternative. This approach has been studied in recent works such as  
\citet*{Lapin2015LossFF} and \citet*{YAN2018479}, which demonstrate its usefulness in certain problems.

List learning has also found applications in theoretical machine learning. For example in \citet*{BrukhimNataly2022ACoM} it was an essential part of establishing the equivalence between finite Daniely-Shwartz (DS) dimension and multiclass learnability. Consequently, list learning has been studied in several recent works in learning theory. For example, \citet*{Charikar2022ACO} characterized list PAC learnability by using a list variant of the DS dimension, and \citet*{Moran:23OnlineList} characterized list online learnability using a list variant of the Littlestone dimension. Another recent application of list learning is in the realm of multiclass boosting; \citet*{Brukhim:23} employed it to devise the first boosting algorithm whose sample complexity is independent of the label space's size.
%Tom: Maybe add some more important works in list learning?

\vspace{1mm}

A natural question that has not yet been systematically addressed is the identification of fundamental principles in list PAC learning. In the binary case, PAC learning is guided by fundamental algorithmic principles, notably Empirical Risk Minimization, and Occam's Razor principles such as compression bounds. In this work, we ask which of these foundational principles remains applicable in the domain of list learning.
% \begin{center}
% \begin{framed}
%     \underline{\bf Guiding Questions}\\
%     \vspace{2mm}
%     What are the basic principles of generalization in list PAC Learning?\\
%     \vspace{2mm}
%     Can every list learnable class be learned by a compression scheme?
%     Does every l
%     % Do principles such as Empirical Risk Minimization and Sample Compression, which are complete\footnotemark\hspace{0mm} in classical PAC learning, retain their completeness in the list setting?
% \end{framed}
% \footnotetext{Here, 'complete' signifies that for any PAC learnable concept class, there exists a learning algorithm conforming to these principles that learns it.}
% \end{center}

% In this work, we study this question in the context of empirical risk minimization and sample compression schemes, 
% which is a strong manifestation of the Occam's Razor principle.
% In fact in the classical PAC setting learnability is equivalent to the existence of finite sample compression schemes~\cite{moranyehudayoof:16,david2016statistical}.

\subsection{Our Contribution}
In this section we summarize our main results. It is based on natural adaptations of basic learning theoretic definitions to the list setting. These definitions are fully stated in \Cref{sec:basicdefs}.

\subsubsection{Sample Compression}
\paragraph{Occam's Razor.}
Occam's razor is a philosophical principle applied broadly across disciplines, including machine learning. It suggests that among competing hypotheses, the simplest one should be selected. In machine learning, this principle is often quantified in terms of the number of bits required to encode an hypothesis, thereby serving as a guideline for preferring simpler models. This concept forms the basis of a general approach in machine learning, where simplicity is directly linked to the efficiency and effectiveness of learning algorithms.

A more refined manifestation of Occam's razor in machine learning is evident in sample compression schemes. These schemes go beyond merely considering the bit-size of an hypothesis. They involve an additional use of a small, representative subset of input examples to encode hypotheses. This approach is a more nuanced application of the principle, allowing for both the simplification of data and the preservation of its essential characteristics.
In fact, while the classical interpretation of Occam's razor, based on bit-encoding, is sound --- implying generalization --- it lacks completeness. This means that there exist learnable classes whose learnability cannot be demonstrated solely through bit-encoding (see \Cref{fig:svm})\footnote{Figures \ref{fig:svm} and \ref{fig:comp} from \citet*{Alon2021ATO}, used with permission.\label{foot:figures}}. In contrast, sample compression schemes offer a more comprehensive manifestation of Occam's razor: every PAC learnable class can be effectively learned by a sample compression algorithm~\citep*{david2016statistical}.

%\footnotetext{}

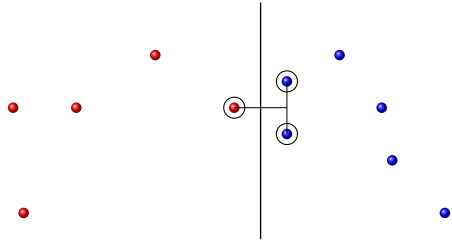
\begin{figure}
\centering
\textbf{\small Support Vector Machine}\par\bigskip   
\begin{tikzpicture}[scale=0.7] 
%\shade[shading=ball, ball color=red]  (8,0.5) circle (.1);
\shade[shading=ball, ball color=red]  (4.8,4) circle (.1);
\shade[shading=ball, ball color=red]  (5,2) circle (.1);
%\shade[shading=ball, ball color=red]  (5,1) circle (.1);
\shade[shading=ball, ball color=red]  (6,4) circle (.1);
\shade[shading=ball, ball color=red]  (9,4) circle (.1);%
\shade[shading=ball, ball color=red]  (7.5,5) circle (.1);
\shade[shading=ball, ball color=blue]  (11.8,4) circle (.1);
\shade[shading=ball, ball color=blue]  (13,2) circle (.1);
%\shade[shading=ball, ball color=blue]  (12,1) circle (.1);
\shade[shading=ball, ball color=blue]  (10,3.5) circle (.1);%
\shade[shading=ball, ball color=blue]  (10,4.5) circle (.1);%
\shade[shading=ball, ball color=blue]  (12,3) circle (.1);
\shade[shading=ball, ball color=blue]  (11,5) circle (.1);
%\shade[shading=ball, ball color=blue]  (11,7) circle (.1);

\v<2->{\draw  (9,4) circle (.2);}
%\v<2->{\draw  (7.5,5) circle (.2);}
%\v<2->{\draw  (8,0) circle (.2);}

\v<2->{\draw  (10,3.5) circle (.2);}
\v<2->{\draw  (10,4.5) circle (.2);}
%\v<2->{\draw  (11,7) circle (.2);}
%\v<2->{\draw  (12,1) circle (.2);}

\v<3->{\path[draw, line width=0.5pt] (9.5,1.5) -- (9.5,6);}
\v<3->{\path[draw, line width=0.1pt] (9,4) -- (10,4);}
\v<3->{\path[draw, line width=0.1pt] (10,3.5) -- (10,4.5);}
\end{tikzpicture}
\caption{\small Support Vector Machine (SVM) as sample compression:
The SVM algorithm identifies a maximum-margin separating hyperplane, which is defined by a subsample of $d+1$ support points. This exemplifies a concept class which cannot be learned with hypotheses specified by finite bit-encoding, yet for which there is an algorithm which learns effectively using hypotheses specified using a bounded number of input-examples.
}
\label{fig:svm}
\end{figure}

\paragraph{List Sample Compression Schemes.}
Sample compression schemes were initially developed by \cite{littlestone:86} for the purpose of proving generalization bounds; however, they can also be interpreted as a standalone mathematical model for \emph{data simplification}. Sample compression resembles a scientist who collects extensive experimental data but then selects only a crucial, representative subset from it. From this subset, a concise hypothesis is formulated that effectively explains the entire dataset. This analogy highlights the essence of sample compression: distilling a complex dataset into a simpler form, while maintaining the capacity to accurately explain the patterns and phenomena of the entire original dataset.

More formally, a compression scheme consists of a pair of functions: a \emph{compressor} and a \emph{reconstructor}. The compressor gets an input sample of labeled examples $S$ and uses it to construct a small subsequence of labeled examples $S'$ which she sends to the reconstructor. The reconstructor uses $S'$ to generate an concept $h$. The goal is that $h$ will be consistent with the original sample~$S$, even on examples that did not appear in $S'$ (see~\Cref{fig:comp})\footref{foot:figures}.

We extend the notion of sample compression to the list setting naturally by changing the concept $h$ outputted by the reconstructor to be a $k$-list map. Now the goal is that $h$ will be consistent with $S$ in the sense that for any example $x$, its label will be one of the $k$ elements of the list $h(x)$. In more detail, a concept class $\C$ is $k$-list compressible if there exists a $k$-list sample compression scheme with a finite size
such that whenever the input sample $S$ is realizable\footnote{I.e.\ there exists $c\in \C$ such that $c(x_i)=y_i$ for all $(x_i,y_i)\in S$} by $\C$ then the reconstructed $k$-list concept $h$ satisfies $y_i\in h(x_i)$ for all $(x_i,y_i)\in S$. We refer to Section~\ref{sec:defcomp} for more details.
 
% \textcolor{red}{Shay: TODO change the definition so that it is more similar to the COLTmain paper. Same remark for the next paragraph.}

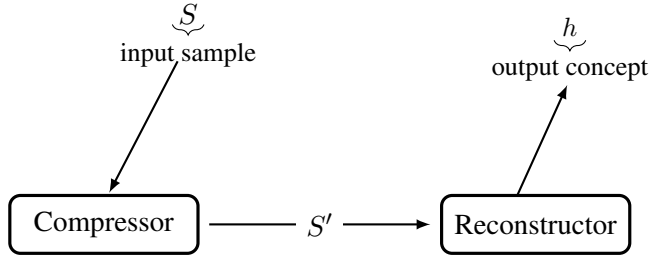
\begin{figure}
\centering
    \textbf{A pictorial definition of a sample compression scheme}\par\bigskip   
\begin{tikzpicture}[scale=0.7]

\node[inner sep=0pt,punkt] (bob) at (7.1,0) {Reconstructor};

\node[inner sep=0pt,punkt] (alice) at (-1.1,0){Compressor};

\draw[->,thick] (0.9,0) -- (5.1,0)
    node[midway,fill=white] {$S'$};

%\draw [decorate,decoration={brace,amplitude=4pt,mirror},yshift=0pt]
%(2.5,-0.2) -- (3.4,-0.2) node [black,midway,yshift=-0.5cm]
%{\small{ a subsample and an additional binary string}};

\node[inner sep=0pt] (sample) at (0.5,4){$S$};
\draw[->,thick] (0.3,3.1) -- (-1,0.6);

\draw [decorate,decoration={brace,amplitude=4pt,mirror},yshift=0pt]
(0.2,3.8) -- (0.8,3.8) node [black,midway,yshift=-0.4cm]
{\small input sample};

\node[inner sep=0pt] (reconstruction) at (7.75,3.75){\small $h$};

\draw [decorate,decoration={brace,amplitude=4pt,mirror},yshift=0pt]
(7.45,3.5) -- (8.05,3.5) node [black,midway,yshift=-0.4cm]
{\small output concept};

\draw[->,thick] (6.75,0.55) -- (7.7,2.65);
\end{tikzpicture}
\caption{An illustration of a sample compression scheme: $S'$ is a subsample of $S$; the output concept $h$ should be consistent with the entire input sample $S$. In the context of $k$-list sample compression, the concept $h$ assigns a list of $k$ possible labels to each data point $x$. The objective is refined such that for every input example $(x_i,y_i)$, the actual label $y_i$ is included within the list provided by $h(x_i)$.}
\label{fig:comp}
\end{figure}

In their original paper \cite{littlestone:86} showed that sample compression schemes are PAC learners, and conjectured that every PAC learnable class admits a sample compression scheme. \cite{moranyehudayoof:16} confirmed this conjecture in the binary setting, and \cite{david2016statistical} further showed that any PAC learnable class over a finite label space admits a sample compression scheme. It is therefore natural to ask whether the analogue  of \cite{littlestone:86} conjecture holds in the list setting:
\begin{center}
    \emph{Does every $k$-list learnable class has a finite $k$-list sample compression scheme?}
\end{center}
Our first result provides a negative answer in the simplest list PAC learning setting:  $2$-list learning a $3$-label space:
\begin{framed}
\vspace{-5mm}
\begin{theorem}[List-Learnability vs.\ List-Compressibility]\label{thm:2-comp}
There exists a concept class $\C$ over the label space $\Y=\{0,1,2\}$ such that:
\begin{itemize}
    \item $\C$ is $2$-list PAC learnable.
    \item $\C$ has no finite $2$-list sample compression scheme.
\end{itemize}    
\end{theorem}
\vspace{-5mm}
\end{framed}

\paragraph{Reconstruction with Larger Lists.}
The class $\C$ in \Cref{thm:2-comp} is trivially $3$-list compressible: simply take  the reconstructed concept $h$ be such that $h(x)=\{0,1,2\}$ for every $x$.
In a related and recent result, \cite{pabbaraju:23} considers the setting of multiclass PAC learning over an infinite label space and establishes the existence of a concept class $\C$ that is (1-list) PAC learnable but not (1-list) compressible.  
Also in his construction, the class $\C$ is trivially $2$-list compressible.
These examples raise the following question:
%\vspace{-2mm}
\begin{center}
\emph{Does every $k$-list PAC learnable class $\C$ admit a $(k+1)$-list sample compression scheme? }
\end{center}
%\vspace{-2mm}
We not only answer this question in the negative but also prove a significantly stronger result. We demonstrate the existence of learnable classes that are not $k$-compressible for any arbitrarily large~$k$:
% encompassing both $k$-list learnability in a finite label space and PAC learnability (or $1$-list learnability) in an infinite label space:
\begin{framed}
\vspace{-5mm}
\begin{theorem}\label{constant compression 2 learnability finite label space}[$2$-Learnability vs $k$-Compressibility]\label{t:2vsk}
For any $k>0$ there exists a concept class~$\C_k$ over a \underline{finite} label space $\Y_k$ that satisfies the following:
    \begin{enumerate}
        \item $\C_k$ is $2$-list PAC learnable. 
        \item $\C_k$ has no finite $k$-list sample compression scheme.
    \end{enumerate}
\end{theorem}
\begin{theorem}\label{constant compression 1 learnability}[$1$-Learnability vs $k$-Compressibility]\label{t:1vsk}
    For any $k>0$ there exists a concept class~$\C_k$ satisfying the following:
    \begin{enumerate}
        \item $\C_k$ is PAC learnable. (I.e.\ $1$-list PAC learnable)
        \item $\C_k$ has no finite $k$-list sample compression scheme.
    \end{enumerate}
\end{theorem}
\vspace{-5mm}
\end{framed}
Note that by \citet*{david2016statistical} the label space of the class in \Cref{constant compression 1 learnability} is inevitably infinite and that the case $k=1$ gives the aforementioned result of \cite{pabbaraju:23}. 
Additionally, a close examination of the proofs for these theorems reveals that in both scenarios, the respective classes lack a compression scheme of size $f(n)$ for any function $f(n)=o(\log n)$. % It is also worth mentioning that those theorems give a tight bound to the above problem since every $1$-list learnable class over a finite label space will be $1$-list compressible. 

The proof of Theorem~\ref{thm:2-comp} is based on a construction from \citet*{Alon2021ATO} of a learnable partial concept class which is not compressible.
This is similar to how \citet*{pabbaraju:23} derives a non compressible PAC learnable class over an infinite label space; 
however, the argument is somewhat more involved because \Cref{thm:2-comp} provides a consturction over a finite label space, whereas in \cite{pabbaraju:23} the label space is inevitably infinite.
The proofs of \Cref{t:2vsk} and \Cref{t:1vsk} are by induction on $k$, where the base case reduces to the partial concept class construction from \citet{Alon2021ATO}. 
The induction step, however, requires new ideas and will use direct sum arguments that we develop in the full version of this paper and overview in \Cref{sec:directsum}.
We give a more detailed review of the proof method in \Cref{proof overview sample compression}.

%TODO: 1. informal definition of list sample compression schemes. Use a figure to illustrate.
%TODO: 2. Question is every list learnable class list compressible? 
% 3. Answer: now, PAB23 recently showed that there are 1-learnable classes thar are not 1-compressible. However in his construction the class is clearly 2-compressible.
%4. This raises the question whether every k-list learnable class is possibly (k+1)-list compressible? Or perhaps even f(k) list compressible for some function f(k)? (Say Ackermann function).
%4. We provide a strong negative answer...

\subsubsection{Uniform Convergence}
Uniform convergence and empirical risk minimization (ERM) are arguably the most extensively studied algorithmic principles for generalization in machine learning. For example, the Fundamental Theorem of PAC Learning for binary classification states the equivalence between PAC learnability, uniform convergence, and the ERM principle~\citep*{shalev-shwartz_ben-david_2014}.
In fact, the equivalence between uniform convergence and PAC learnability holds whenever the label space is finite~\citep*{JMLR:v11:shalev-shwartz10a}. Moreover, ERM is closely related to other statistical principles such as maximum likelihood estimation.

Informally, uniform convergence refers to the phenomenon where, given a sufficiently large sample from a population, the empirical losses of all concepts in a class closely approximate their true losses over the entire population. This phenomenon forms the basis of the Empirical Risk Minimization principle. It posits that selecting a concept in the class that minimizes empirical loss is a sound strategy, as such a concept is also likely to approximately minimize the population loss.

In the context of list learning, we consider classes $\C$ of $k$-list concepts. So, each $c\in \C$ assigns a list of $k$ labels to each point $x\in \X$. For a target distribution $\D$ over $\X\times \Y$, the population loss of $c$ is $\mathtt{L}_\D(c)=\E_{(x,y)\sim \D}[1_{y\notin c(x)}]$ and for a sample $S=\{(x_i,y_i)\}_{i=1}^n$, the empirical loss of $c$ is $\mathtt{L}_S(c)=\frac{1}{n}\sum_{i=1}^n1[y_i\notin c(x_i)]$.
% We denote such classes by $\H$ and call them \emph{hypothesis} classes; this is in contrast with the notation $\C$ which is used for concept classes and is meant to highlight the distinction that a concept class $\C$ is interpreted as one which contains the concept we wish to learn, whereas an hypothesis class $\H$ is the set of hypotheses used by the learner.
% In the case of hypothesis classes $\H$, the target distribution of examples remains a distribution on $\X \times \Y$, so each example~$x$ has one correct label $y$, and each function $h\in \H$ is considered correct on $x$ if $y$ is included in the list~$h(x)$.
% The standard realizability assumption from PAC learning is then generalized to this setting by supposing there is a function $h$ in the class with zero population loss.
We investigate when $k$-list concept classes exhibit uniform convergence and ask whether a parallel to the Fundamental Theorem of PAC Learning exists in this setting. 
\begin{center}
    \emph{Does the equivalence between PAC learning and uniform convergence extend to $k$-list concept classes?}
\end{center}
We confirm this is the case, providing an affirmative response to these questions, though via a novel analysis deviating significantly from the traditional approach (as outlined below).
Please see Section~\ref{sec:defgen} for the definitions of PAC learnability and uniform convergence, adapted to the list learning setting.
\begin{framed}
\vspace{-5mm}
\begin{theorem}[List Learnability vs.\ Uniform Convergence]\label{t:UC}
Let $\C\subset {\Y \choose k}^\X$ be a $k$-list concept class over a finite label space $\lvert \Y\rvert<\infty$. Then, the following properties are equivalent:
\begin{itemize}
    \item $\C$ is $k$-list PAC learnable.
    \item $\C$ is $k$-list agnostically PAC learnable.
    \item $\C$ satisfies the uniform convergence property.
\end{itemize}
\end{theorem}
\vspace{-5mm}
\end{framed}
% \begin{remark}
% In the context of uniform convergence and \Cref{t:UC} we use the notation $\H$ to denote $k$-list function classes and call them \emph{hypothesis classes}; this is in contrast with \Cref{thm:2-comp,t:2vsk,t:UC} where we use the notation $\C$ and use the name \emph{concept classes}. This distinction is meant to highlight that a concept class~$\C$ is interpreted as one which contains the concept we wish to learn/compress, whereas an hypothesis class~$\H$ is the set of hypotheses used by an empirical risk minimizer. See also Remark~\ref{rmk:conceptvslisthypotheses} below.
% \end{remark}
In \Cref{graph dimension vs DS dimension} in Section~\ref{sec:Uniform Convergence Proofs} we also provide quantitative bounds on the uniform convergence rate. These bounds follow from analyzing the $\DS_k$ dimension (which controls the learning rate) and the graph dimension (which controls the uniform convergence rate); see \Cref{sec:basicdefs} for further details. In particular, our result implies 
$\DS_k(\C) =  \tilde\Omega\bigl(\frac{\Gk(\C)}{k^2\cdot\log(\lvert \Y\rvert)+k\log \Gk(\C)}\bigr)$. This implies the following upper bound on the uniform convergence rate: 
\begin{align*}
    \varepsilon(n\vert \C) = \tilde O\Biggl(\sqrt{\frac{k^2\cdot \DS_k(\C)\cdot \log(\lvert \Y\rvert\cdot \DS_k(\C))}{n}}\Biggr).
\end{align*}
This finding extends the ERM principle in the realm of list learning: for any class of $k$-list concepts that is $k$-list learnable, an effective learning strategy is to choose a concept from the class that minimizes the empirical loss.

The assumption that the label space $\Y$ is finite is necessary. Indeed, \citet*{daniely:15} demonstrate that already in the case of $k=1$, there are PAC learnable classes that do not satisfy uniform convergence.

The proof of \Cref{t:UC} deviates from the classical approaches to deriving uniform convergence. Typically, these bounds are obtained using a ghost sample argument combined with a growth function bound for the concept class $\C$. However, in the $k$-list learning context, some list-learnable classes $\C$ exhibit growth functions that are excessively large to apply this method. To overcome this, we directly analyze the VC dimension of the loss functions. Utilizing a probabilistic argument, we demonstrate that a high VC dimension of the loss function directly implies a significantly large $k$-DS dimension for the class $\C$. Consequently, if the class $\C$ does not satisfy uniform convergence, it is not PAC learnable. 

\subsubsection{Direct Sum}\label{sec:directsum}
In computer science, the term 'direct sum' refers to fundamental questions about the scaling of computational or information complexity with respect to multiple task instances. Consider an algorithmic task \( T \) and a computational resource \( C \). For instance, \( T \) might be the task of computing a polynomial, with \( C \) representing the number of arithmetic operations required, or \( T \) could be a learning task with its sample complexity as \( C \). The direct sum inquiry focuses on the cost of solving~\( k \) separate instances of \( T \), particularly how this aggregate cost compares to the resources needed for a single instance. Typically, the cost for multiple instances is at most \( k \) times the cost of one, since each can be handled independently. 

However, there are intriguing scenarios where the total cost for \( k \) instances is less than this linear relationship. These cases suggest more efficient methods for simultaneously handling multiple instances of a task than addressing them one by one. As an example, consider an \( n \times n \) matrix \( A \) and the objective of calculating its product with an input column vector \( x \), where the computational resource \( C \) is the number of arithmetic operations. For a single vector \( x \), it is easy to see that \(\Theta(n^2)\) operations are necessary and sufficient. However, if  instead of one input vector $x$, there are $n$ input vectors $x_1,\ldots, x_n$ then one can do better than $n\times \Theta(n^2)= \Theta(n^3)$. Indeed, by arranging these \( n \) vectors as columns in an \( n \times n \) matrix \( B \), computing the product \( A \cdot B \) is equivalent to solving the~\( n \) products \( A \cdot x_i \). This task can be accomplished using roughly \( n^\omega \leq n^{2.37} \) arithmetic operations with fast matrix multiplication algorithms. Direct sum questions are well-studied in information theory and complexity theory. For more background we refer the reader to the thesis by~\citet*{pankratov2012direct} or the books by \cite{wigderson:19} and \citet*{Rao_Yehudayoff_2020}.

\paragraph{Direct Sum in Learning Theory.} Natural direct sum questions can also be posed in learning theory. To formalize these, we use the notion of cartesian product of concept classes: consider two concept classes, \( \C_1 \) and \( \C_2 \), defined over domains \( \X_1 \) and \( \X_2 \), and label spaces \( \Y_1 \) and \( \Y_2 \) respectively. Their product, \( \C_1 \otimes \C_2 \), has domain \( \X_1 \otimes \X_2 \) and label space \( \Y_1 \otimes \Y_2 \). Each concept \( c \) in \( \C_1 \otimes \C_2 \) is parameterized by a pair of concepts \( c_1 \in \C_1 \) and \( c_2 \in \C_2 \), and is defined as \( c((x_1, x_2)) = (c_1(x_1), c_2(x_2)) \). Thus, learning \( c \) effectively means learning both \( c_1 \) and \( c_2 \) simultaneously.

In our proofs of Theorems~\ref{t:2vsk} and~\ref{t:1vsk}, we study the sample compression complexity of such product classes. While direct sum analysis serves primarily as a technical instrument in our research, it also leads to basic questions that we propose for future research. For instance, consider the following question: 

\begin{framed}
\vspace{-3mm}
\begin{open}[Direct Sum: PAC Learning Curves]\label{open:sumofcurves}
Let $\C\subseteq \Y^\X$ be a concept class, and let $\eps(n \vert \C)$ denote the realizable PAC learning curve of $\C$ (see Definition \ref{def:PAC}). 
For $r\in\mathbb{N}$ let $\C^r= \Pi_{i=1}^r \C$ be the $r$-fold Cartesian product of $\C$.
By a union bound, learning each component independently gives
    \[\eps(n \vert \C^r) \leq  r\cdot \eps(n \vert \C).\]
Can the upper bound be asymptotically improved for some classes $\C$?
\end{open}
\vspace{-3mm}
\end{framed}
Further discussion and open questions related to the direct sums of learning problems are elaborated in the full version of this paper.

\paragraph{Organization.} 
The remainder of this paper is organized as follows: \Cref{sec:basicdefs} presents the fundamental definitions of PAC learnability and sample compression, adapted for the list learning setting. Subsequently, \Cref{sec:overview} provides an overview of the techniques and key ideas employed in our proofs. 
% Finally, \Cref{Direct Sum and Open Questions} outlines several open questions for future research, particularly focusing on the direct sum context.
% The complete proofs as well as several open questions appear in the full version of this paper and can be found in the supplementary material.
Then, in \Cref{sec:Sample Compression Proofs}  we give proofs for results on sample compression, and in \Cref{sec:Uniform Convergence Proofs}  the uniform convergence results are proven. Finally in  \Cref{sec:Direct Sum and Open Questions} we look deeper into the direct sum questions in learning theory and propose some open questions and directions for future research.

\section{Basic Definitions}\label{sec:basicdefs}
\subsection{Generalization}\label{sec:defgen}
We use standard notation from learning theory, see e.g.\ \cite{shalev-shwartz_ben-david_2014}. 
    Let $\X$ denote the domain and $\Y$ denote the label space. 
    A $k$-list function (or $k$-list concept) is a function $c:\X\to {\Y \choose k}$, where ${\Y \choose k}$ denotes the collection of all subsets of~$\Y$ of size $k$. 
    A $k$-list concept class $\C\subseteq {\Y \choose k}^\X$ is a set of $k$-list functions.
    Note that by identifying sets of size one with their single elements, $1$-list concept classes correspond to standard concept classes. 
 
A $k$-list learning rule is a map $\A: (\X\times \Y)^*\to {\Y \choose k}^\X $, 
    i.e.\ it gets a finite sequence of labeled examples as input and outputs a $k$-list function. 
    A learning problem $\D$ is a distribution over $\X\times \Y$.
     The population loss of a $k$-list function $c$ with respect to $\D$ is defined by $\mathtt{L}_\D(c)=\E_{(x,y)\sim \D}[1_{y\notin c(x)}]$.

We quantify the learning rate of a given learning rule on a given learning problem using learning curves:
\begin{definition}[Learning Curve]\label{def:learningcurve}
    The learning curve of a learner $\A$ with respect to a learning problem $\D$ is the sequence $\{\eps_n(\D \vert \A)\}_{n=1}^\infty$,
    where 
    \[\eps_n(\D \vert \A)= \E_{S\sim \D^n}\Big[ \mathtt{L}_{\D}\big(\A(S)\big)\Big].\]
  In words, $\eps_n(\D \vert \A)$ is the expected error of the learner $\A$ on samples of size $n$ drawn from the distribution~$\D$. 
  % It quantifies the probability that the learner misclassifies the label $y$ for a given input $x$.
\end{definition}
For a sequence $S$ of labeled examples, the empirical loss of a $k$-list function $c$ with respect to $S$ is 
$\mathtt{L}_S(c)=\frac{1}{|S|}\sum_{(x,y)\in S} 1_{y\notin c(x)}$.
A sequence $S\in (\X\times \Y)^n$ is realizable by a $k$-list function $c$ if $y\in c(x)$ for every $(x,y)\in S$.
It is realizable by a concept class $\C$ if it is realizable by some concept $c\in \C$.
A learning problem $\D$ is realizable by a concept class $\C$ if for any $n\in \mathbb{N}$, a random sample $S\sim \D^n$ is realizable by $\C$ with probability $1$.
\begin{definition}[List PAC Learnability]\label{def:PAC}
%Let $k\in\mathbb{N}$. 
We say that a concept class $\C$ is agnostically $k$-list learnable
if there exists a $k$-list learning rule $\A$ and a sequence $\varepsilon_n\xrightarrow{n\to\infty} 0$ such that for every learning problem $\D$,
$(\forall n): \eps_n(\D\vert A) \leq \inf_{c\in\C}\mathtt{L}_\D(c) + \eps_n$.
%and every $n$: $\eps_n(\D\vert A) \leq \inf_{c\in\C}\mathtt{L}_\D(c) + \eps_n$.
    If the latter only holds for $\C$-realizable distributions then we say that $\C$ is $k$-list learnable in the realizable case.
The $k$-list realizable PAC learning curve of a concept class $\C$ is defined as follows:
  \[
    \eps(n \vert \C) = \inf_{\A}\sup_{\D}\eps_n(\D\vert A),
  \]
  where the infimum is taken over all $k$-list learning rules $\A$ and the supremum over all distributions~$\D$ that are realizable by $\C$.
\end{definition}

% \begin{definition}[List PAC Learnability]\label{def:PAC}
% Let $k,k'\in\mathbb{N}$. 
% We say that a $k'$-list concept class $\C$ is agnostically $k$-list learnable
% if there exists a $k$-list learning rule $\A$ and a sequence $\varepsilon_n\xrightarrow{n\to\infty} 0$ such that for every learning problem $\D$ and every $n$: $\E_{S\sim \D^n}\Big[ \mathtt{L}_D\big(\A(S)\big)\Big]\leq \inf_{c\in\C}\mathtt{L}_\D(c) + \eps_n$.
%     If the latter only holds for $\C$-realizable distributions then we say that $\C$ is $k$-list learnable in the realizable case.

% The $k$-list realizable PAC learning curve of a concept class $\C$ is defined as follows:
%   % \[
%   %   \eps(n \vert \C) = \inf_{\A}\sup_{\D}\E_{S\sim \D^n}\Big[ \mathtt{L}_D\big(\A(S)\big)\Big],
%   % \]
%   $\eps(n \vert \C) = \inf_{\A}\sup_{\D}\E_{S\sim \D^n}\Big[ \mathtt{L}_D\big(\A(S)\big)\Big]$,
%   where the infimum is taken over all $k$-list learning rules $\A$ and the supremum over all distributions~$\D$ that are realizable by $\C$.
% \end{definition}

%Note that in this definition $k'$ is not necessarily equal to $k$. 
%Thus, for instance, a $(k'=1)$-concept class can still be $(k=2)$-list PAC learnable.
Observe that a class $\C$ is $k$-list PAC learnable in the realizable case if and only if its PAC learning curve approaches zero as $n\to\infty$. 
\begin{remark}
In the literature, PAC learnability is sometimes defined with a more stringent requirement: that the error is small with high probability, as opposed to merely the expected error in our definition. These two formulations are equivalent, as follows by a standard confidence amplification technique. We chose the above definition as it is simpler in that it involves fewer parameters (it omits the confidence parameter).
\end{remark}

In the full version of the paper we prove an equivalence between agnostic and realizable case learnability:
\begin{theorem}\label{t:realvsagn}
    Let $\C$ be a concept class. Then $\C$ is $k$-list learnable in the realizable setting if and only if it is $k$-list learnable in the agonstic setting.
\end{theorem}
Hence, we sometimes refer to a concept class as 'learnable' without distinguishing between the agnostic and realizable cases. The proof of \Cref{t:realvsagn} is a straight-forward adaptation of the proof of the parallel result for $k=1$~\cite{david2016statistical}. %Consequently, we can

% We quantify the optimal PAC learning rate using PAC learning curves:
% \begin{definition}[PAC Learning Curve]\label{def:PACurve}
  
% \end{definition}

Our last definition in this section is of uniform convergence.
\begin{definition}[Uniform Convergence]\label{def:UC}
We say that a $k$-list concept class $\C$ satisfies uniform convergence if there exists a vanishing sequence $\eps_n\xrightarrow{n\to\infty} 0$ such that for all distributions $\D$,
\[\E_{S\sim \D^n}[\sup_{h\in \C}\lvert \mathtt{L}_\D(h) - \mathtt{L}_S(h) \rvert]\leq \eps_n.\]
The uniform convergence rate of $\C$ is defined by
$\eps_{\mathtt{UC}}(n \vert \C) = \sup_{\D}\E_{S\sim \D^n}[\sup_{h\in \C}\lvert \mathtt{L}_\D(h) - \mathtt{L}_S(h) \rvert]$,
where the supremum is over all distributions $\D$.
\end{definition}
Note that $\C$ satisfies uniform convergence if and only if its uniform convergence rate converges to zero as $n\to\infty$.
It is also worth mentioning that in the binary setting (i.e.\ when $\Y=\{0,1\}$), PAC learnability is equivalent to uniform convergence,
and that this equivalence justifies the Empirical Risk Minimization principle of outputting an concept in $\C$ which minimizes the sample error.

% We say that a  $k$-list function class $\F\subset {\Y\choose k}^\X$ is $k$-list PAC learnable if $(\F,{\Y\choose k}^\X)$ is list learnable and that $\F$ is $k$-list agnostic learnable if $(\Y^\X,\F)$ is list learnable. 

\subsection{List Sample Compression}\label{sec:defcomp}
%We formally define list compressibility by  
\begin{definition}[List Compressibility]
%Let $k,k'\in\mathbb{N}$. 
A concept class $\C$ is agnostically $k$-list compressible
if there exist $d\in\mathbb{N}$ and a reconstruction function $\rho:(\X\times\Y)^d\to {\Y \choose k}^\X$ such that the following holds.
For every sample $S\in(\X\times\Y)^n$ there exists $S'=\big((x_1,y_1),\ldots (x_{d},y_{d})\big)$, where $(x_i,y_i)\in S$ for all $i\leq d$ such that
    % \begin{equation}\label{eq:compressibility}
    % \mathtt{L}_S(\rho(S')) \leq \inf_{c\in \C}\mathtt{L}_S(c).
    % \end{equation}
    $\mathtt{L}_S(\rho(S')) \leq \inf_{c\in \C}\mathtt{L}_S(c)$.
    If the latter only holds for $\C$-realizable samples then we say that $\C$ is $k$-list compressible in the realizable case. 
\end{definition}
In some places, the map that takes $S$ to $S'$ is explicitly defined as the compression map and is denoted by $\kappa$, and the pair $(\rho,\kappa)$ is called the sample compression scheme.

% \begin{definition}
% Let $f:\mathbb{N}\to\mathbb{N}$. We say that $\mathcal{C}$ has a $k$-list sample compression scheme of variable size $f$ if there exists a reconstruction function $\rho : (\X\times \Y)^* \to {\Y \choose k}^\X$ such that for any realizable sample $S$ of size $n$ there is some $\kappa(S)=\big((x_1,y_1),(x_2,y_2),\dots (x_{f(n)},y_{f(n)})\big)$ such that $\mathtt{L}_S(h)=0$, where $h=\rho(\kappa(S))$.
% The map $\kappa$ that takes $S$ to $\kappa(S)$ is called the compression map and the pair $(\rho,\kappa)$ is called a compression scheme.
% \end{definition} 

% We would like to emphasize that, similar to the case of learnability, the concept class $\C$ may be a $k'$ concept class for any $k'$(specifically $k'$ need not be equal to $k$ ). However, our primary focus will be on the standard case where $k'=1$.
% Furthermore, in some definitions of sample compression, side information is explicitly allowed. The above definition allows to simulate side information by using reordering and repeating examples.

In the full version of this paper we include additional basic results on list sample compression schemes. In particular, we prove that any list sample compression scheme generalizes and that learnability is equivalent to \emph{logarithmic}-compressibility.
The latter is a variant of the above definition of sample compression where the size of the compression scheme is not a fixed constant $d$, but rather depends logarithmically on the size of the input sample.

% \begin{remark}[Functions vs.\ List Functions]\label{rmk:conceptvslisthypotheses}
% Let $\C$ be a $k$-list concept class. Define a function class 
% \[\F=\F(\C)=\{f:\X\to\Y : f\dis c \text{ for some }c\in\C\},\] 
% where $f\dis c$ means $f(x)\in c(x)$ for all $x\in\X$. Note that $\F$ and $\C$ are equivalent in the sense that a sample~$S$ is realizable by $\F$ if and only if it is realizable by $\C$. In particular $\F$ is $k$-learnable ($k$-compressible) if and only if $\C$ is. Thus, when discussing learnability (or compressibility), we may restrict our attention to ($1$-list) function classes without losing generality. For this reason in \Cref{thm:2-comp,t:2vsk,t:1vsk} we can focus on function classes.

% In contrast, this reduction from list-functions to functions does not make sense when studying empirical risk minimization over a $k$-list concept class $\C$. Moreover, the above reduction from $\C$ to $\F(\C)$ does not preserve uniform convergence. Indeed, let $\C=\{c\}$, where $c$ is the $2$-list function such that $c(x)=\{0,1\}$ for all $x$.
% Clearly, $\C$ satisfies uniform convergence (because it is finite), however $\F(\C)=\{0,1\}^\X$ does not satisfy uniform convergence when $\X$ is infinite.
% For this reason in \Cref{t:UC} we focus on $k$-list function classes.
% \end{remark}

\subsection{Combinatorial Dimensions}
We next introduce the Graph and Daniely-Shwartz (DS) dimensions which are known to characterize Uniform Convergence and List PAC learnability.
\begin{definition}[Graph Dimension]
    A sequence $S=\{x_i\}_{i=1}^n$ is $\Gk$-shattered by a $k$-list concept class $\C$ if there is  $p\in \Y^n$ called a piovt, such that for any $b\in \{0,1\}^n$ there is $c_b\in \C$ such that $1[p_i\in c_b(x_i)]=b_i$. In other words the (binary) concept class $\{1_{p_i\in c(x_i)} : c\in \C\}$ shatters $S$ in the classical VC sense of shattering. The $k$-graph dimension of $\C$ is $\Gk(\C)$ the size of the largest $\Gk$-shattered sequence, or infinity if there are $\Gk$ shattered sequences of arbitrary size.
\end{definition}
  It is well known that the graph dimension characterizes uniform convergence in the following sense  
\begin{theorem}[\citet*{Daniely2011MulticlassLA}]\label{t:finite graph dimension vs unifrom convergence}
A $k$-list concept class $\C$ satisfies uniform convergence if and only if $\Gk(\C)<\infty$. Furthermore, the uniform convergence rate of $\C$ is $\Theta\bigl(\sqrt{\frac{\Gk(\C)}{n}}\bigr)$.
\end{theorem}

  % \begin{theorem}[\citet*{Daniely2011MulticlassLA}]
  % The following are equivalent for a $k$-list hypothesis class $\C$:
  % \begin{itemize}
  %     \item $\C$ satisfies uniform convergence.
  %     \item $\Gk(\C)<\infty$
  % \end{itemize}
  %     Furthermore, the uniform convergence rate of $\C$ is $\Theta\Bigl(\sqrt{\frac{\Gk(\C)}{n}}\Bigr)$.
  % \end{theorem}

\paragraph{Daniely-Shwartz Dimension.}

To define the DS dimension we first need to introduce the concept of a pseudo-cube, which is a generalization of the boolean hypercube $\{0,1\}^n$ when the label space is $\Y=[m]$.
We say that $y,y'\in [m]^n$ are neighbors in direction $i$ (or simply $i$-neighbors) if $y_j=y_j'$ if and only if $j\neq i$.
We say that $B\subset \Y^n$ is a pseudo-cube of rank $d$ if each $y\in B$ has at least $d$ distinct neighbors in each direction $i\in [n]$.
\begin{definition}[DS Dimension]
    A sequence $S=\{x_i\}_{i=1}^n$ is $\DS_k$-shattered by $\C$ if the set $\{(y_1,y_2\dots y_n): \exists c\in \C,\; \forall i \;\; y_i\in c(x_i)\}$ contains a pseudo-cube of rank $k$.
    The $k$-DS dimension of $\C$ is $\DS_k(\C)$ the size of the largest $\DS_k$ shattered sequences, or infinity if there are $\DS_k$-shattered sequences of arbitrary size.
\end{definition}
 As shown by \cite{Charikar2022ACO} the $\DS_k$ is the combinatorial dimension that characterizes $k$-list learnability in the following sense
\begin{theorem}[\cite{Charikar2022ACO}]\label{t:finite DS-k vs learnability }
  A concept class $\C$ is $k$-list PAC learnable if and only if $\DS_k(\C)<\infty$.
\end{theorem}
We remark that in \cite{Charikar2022ACO} the above was shown only in the case where $\C$ is a function class, but the proofs work in the exact same way in the case that $\C$ is a $k$-list function class.
% \begin{theorem}[\cite{Charikar2022ACO}]
%   The following are equivalent for a concept class~$\C$:
%   \begin{itemize}
%       \item $\C$ is $k$-list PAC learnable.
%       \item $\DS_k(\C)<\infty$
%   \end{itemize}
%       % Furthermore, when the label space $\Y$ is finite, the PAC learning curve of $\C$ is $O\Bigl({\frac{\\DS_k(\C)\log\lvert Y\rvert}{n}}\Bigr)$.
%   \end{theorem}

\vspace{-3mm}

\section{Technical Overview}\label{sec:overview}
 In this section, we overview the main ideas which are used in the proofs.
	We also try to guide the reader on which of our proofs reduce to known arguments and which require new ideas.
 \subsection{Sample Compression Schemes}\label{proof overview sample compression}
%  \textcolor{red}{TOM: The notion of coverability is mostly used implicitly in the paper. We only define it here and in the open questions part and in the main result we work directly with the covering size. Maybe we want to change this? also maybe we want a different name than coverable/covering size?}
%  \textcolor{red}{Shay: not sure, it sounds ok to me but I need to read more carefully\ldots}
% \textcolor{red}{TOM: Originally I used the names canonical disambiguation and trivial disambiguation since the trivial disambiguation is the most natural way to define it and the canonical disambiguation preserves all the properties of the original class. This, however, seems to suggest that one of those is the more correct and the other is less important, which was true in our original work but is no longer the case, maybe we should change the names? (not sure to what though) }
% \textcolor{blue}{Shay: perhaps lets use the notation ``free disambiguation'' for the disambiguation that assigns a distinct label to each function and the notation ``minimal disambiguation'' for the disambiguation that assigns the same label to all functions?}
Theorems \ref{thm:2-comp}, \ref{t:2vsk}, and \ref{t:1vsk} provide impossibility results for sample compression schemes. These types of results are relatively uncommon in the literature, underscoring the technical challenges involved in comprehensive reasoning about all sample compression schemes. We circumvent these challenges by defining a simpler combinatorial notion of \emph{coverability}, which is implied by compressibility. In essence, if a small-sized sample compression exists, then small covers must also exist, and conversely, the lack of small covers indicates the absence of such a compression scheme.

\subsubsection{Coverability}
A $k$-list concept class $\H$ is a $k$-list cover of a ($1$-list) function class $\C$ if for any $c\in \C$ there is $h\in \H$ such that $c(x)\in h(x)$ for all $x\in \X$. We say that $\C$ is $k$-list coverable if there is some polynomial $p$ such that for any finite $S\subset \X$, the finite class $\C\vert_S = \{c\vert_S : c\in\C\}$ has a $k$-list cover of size $p(\lvert S\rvert)$. (Recall that $c\vert_S$ denotes the restriction (or projection) of the function $c$ to the set $S$.)

This generalizes the notion of \emph{disambiguation of partial concept classes}:
A partial concept class $\C$ is a class of partial functions $c:\X\to\Y\cup\{\star\}$, where $\star\notin \Y$ and $c(x)=\star$ means that $c$ is undefined on $x$. A class $\H$ is said to disambiguate $\C$ if for every $c\in\C$ there exists $h\in \H$ such that $h(x)=c(x)$ whenever $c(x)\neq\star$. Notice that $1$-covers are equivalent to disambiguations.
Partial concept classes and disambiguations were studied by~\citep*{Long01,Attias:22,Alon2021ATO,HatamiHM23,Cheung:23}.

In \Cref{compression gives small growth rate} we prove that if $\C$ is $k$-list compressible then it is $k$-list coverable.
Thus, to show that $\C$ is not $k$-list compressible it suffices to show that it is not $k$-list coverable. Given this reduction, the proof can be divided into the following steps:
\begin{itemize}
    \item[(i)] By \cite{Alon2021ATO} there is a partial concept class $\C$ that is $1$-list learnable but not $1$-list coverable.
    \item[(ii)] Boosting the hardness of $\C$: by a direct sum argument we show that the $k$-fold power $\C^k$ is $1$-list learnable but not $k$-list coverable.   
    \item[(iii.a)]  \Cref{thm:2-comp,t:2vsk} follow\footnote{For \Cref{thm:2-comp} we just use the case of $k=1$.} by taking the \emph{minimal disambiguation} of $\C^k$ (defined below). 
    \item[(iii.b)] \Cref{t:1vsk} follows by taking the \emph{free disambiguation} of $\C^k$ (defined below).
\end{itemize}
% The first part is just a direct application of \cite{Alon2021ATO} with the fact that a $1$-list cover of a partial concept class is just disambiguation of it. The second part is our main use of direct sums and relies on Lemma \ref{prodCover}. Which is the quantitative version of the claim that if $\F$ is not $k$-list coverable and $\G$ is not $k'$-list coverable then $\F\otimes\G$ is not $(k+k')$-list coverable. 
The first two steps yield a partial concept class that is $1$-list learnable but not $k$-list compressible. The next two steps are parallel to each other, these steps employ two types of disambiguations which complete the partial concept class to a total concept class in two ways that witness \Cref{t:2vsk,t:1vsk}.

\subsubsection{Free Disambiguations}
\begin{definition}[Free Disambiguation]  \label{def:Free Disambiguation}  
    Let $\C$ be a partial concept class. For each $c\in \C$ let $y_c$ be a distinct new label. Let $\hat c$ denote the completion of $c$ such that $\hat c(x)=y_c$ whenever $c(x)=\star$. 
    % Define 
    % \[
    % \hat c(x)=
    % \begin{cases}
    % c(x)    &c(x)\neq\star\\
    % y_c     &c(x)=\star.
    % \end{cases}  
    % \]
    The class $\Hat\C = \{\hat{c} : c\in \C\}$ is called the \underline{free disambiguation} of $\C$.
\end{definition}
That is, each function in $\C$ is disambiguated by replacing all instances of $\star$ with a unique label for that function. 
%Note that if $\C$ is infinite then $\Hat\C$ has an infinite label space. 
The following lemma is the key to step (iii.b):
\begin{lemma}\label{lem:freedis}
    A partial concept class $\C$ is $k$-list learnable (compressible) if and only if its free disambiguation $\hat{\C}$ is $k$-list learnable (compressible)
\end{lemma}
We note that the free disambiguation was used in \cite{pabbaraju:23} on the partial concept class by \cite{Alon2021ATO} to establish the existence of a learnable class that is not compressible, here we apply it more generally to the classes generated by the direct sum argument.

\subsubsection{Minimal Disambiguations}
Since the free disambiguation invetiably has an infinite label space, it cannot be used to derive Theorems \ref{thm:2-comp} and \ref{t:2vsk}. For that, we introduce a different type of disambiguation:
\begin{definition}[Minimal Disambiguation]\label{def:Minimal Disambiguation}
    Let $\C$ be a partial concept class and let $y_{\star}$ be a new label. For a partial concept $c$, let 
    $\bar c$ denote the completion of $c$ such that $\bar c(x)=y_\star$ whenever $c(x)=\star$.
    % \[
    % \bar c(x)=
    % \begin{cases}
    % c(x)    &c(x)\neq\star\\
    % y_\star     &c(x)=\star.
    % \end{cases}  
    % \]
    The class $\Bar\C = \{\bar{c} : c\in \C\}$ is called the \underline{minimal disambiguation} of $\C$.
\end{definition}
So, all instances of $\star$ are disambiguated by the same new label. 
Here, the label space of $\Bar\C$ has just one more label than that of $\C$.
In particular, the label space of $\Bar\C$ is finite whenever the label space of $\C$ is finite. The following lemma is the key to step (iii.a):
\begin{lemma}\label{lem:mindis}
    Let $\C$ be a partial concept class over a finite label space and let $\Bar{\C}$ be its minimal disambiguation. Then, (i) if $\C$ is $k$-list learnable then $\Bar{\C}$ is $(k+1)$-list learnable, and (ii) if $\C$ is $k$-list learnable and $\Bar{\C}$ is $(k+1)$-list coverable  then $\C$ is $k$-list coverable.
\end{lemma}
% \begin{lemma}\label{lem:mindis}
%     Let $\C$ be a partial concept class over a finite label space and let $\Bar{\C}$ be its minimal disambiguation. Then,
%     \begin{itemize}
%         \item If $\C$ is $k$-list learnable then $\Bar{\C}$ is $(k+1)$-list learnable.
%         \item If $\C$ is $k$-list learnable and $\Bar{\C}$ is $k+1$-list coverable  then $\C$ is $k$-list coverable.
%     \end{itemize}
% \end{lemma}
The above lemma is significantly more nuanced then Lemma~\ref{lem:freedis}. 
In particular it replaces compressibility with coverability and provides implications in one direction rather than equivalences. 
Nevertheless it suffices to yield \Cref{thm:2-comp,t:2vsk}: we apply it to the basic partial concept class from \citep{Alon2021ATO} to deduce \Cref{thm:2-comp}, and to the classes generated by the direct sum argument to deduce \Cref{t:2vsk}.
It would be interesting to find a simpler and more direct argument like in Lemma~\ref{lem:freedis}.

\subsection{Uniform Convergence}
\Cref{t:UC} states an equivalence between learnability, agnostic learnability, and uniform convergence for $k$-list concept classes $\C$. It is clear that uniform convergence implies agnostic learnability, which implies learnability. Thus, it remains to show that learnability implies uniform convergence. Towards this end it suffices to show that if the graph dimension $\Gk(\C)$ (which controls uniform convergence) is unbounded then also the DS dimension $\DS_k(\C)$ (which control list learnability) is unbounded. 

Assume $S=\{x_i\}_{i=1}^n$ is $\Gk$-shattered by $\C$; thus there exists a pivot $p\in \Y^S$ and functions $\{c_b\}_{b\in\{0,1\}^S}$ such that $p_i\in c_b(x_i)$ if and only if $b_i=1$. For each $c_b$ let $A_b:=\{y\in \Y^S\;:\;\forall i\; y_i\in c_b(x_i)\}$ denote its set of realizable sequences. 
To lower bound $\DS_k(\C)$ we lower bound the size of the union $\lvert\bigcup_{b\in\{0,1\}^n}A_b\rvert$ and apply a version of the Sauer–Shelah-Perles lemma from \cite{Charikar2022ACO}. Our lower bound on $\lvert\bigcup_{b\in\{0,1\}^n}A_b\rvert$ is based on a coding theoretic approach and consists of the following steps:
% To prove \Cref{t:UC} we wish to show that if 
%  graph dimension is large so is the number of realizable patterns, then we can apply a version of the Sauer–Shelah lemma from \cite{Charikar2022ACO} to deduce that the $\DS_k$ is large. This approach naturally lead us to consider $k$-list functions as sets of realizable sequences. Formally, given $k$-list function class $\F$ and a  $\F$-shattered set $S=\{x_i\}_{i=1}^n$ we have some pivot $p\in \Y^S$ and functions $\{f_b\}_{b\in\{0,1\}^S}\subset \F\vert_S$ such that $p_i\in f_b(x_i)$ iff $b_i=1$. Then for each $f_b$ we denote its set of realizable functions by $A_b:=\{y\in \Y^S\;:\;\forall i\; y_i\in f_b(x_i)\}$. With notation we can summarise our approach in the following way 
 \begin{enumerate}
     \item We first upper bound the size of the intersection $\lvert A_b\cap A_{b'}\rvert$ in terms of the Hamming distance between $b$ and $b'$.
     \item We then utilize the above within an inclusion-exclusion bound on \(\{A_b\}_{b\in R }\), where $R\subseteq\{0,1\}^n$ is a random subset of the cube. This yields a lower bound on \(\lvert\bigcup_{b\in R}A_b\rvert\) and hence also on  \(\lvert\bigcup_{b\in\{0,1\}^n}A_b\rvert\).
     % \item  We apply the Sauer–Shelah lemma from \cite{Charikar2022ACO} to lower bound $\DS_k(\C)$
 \end{enumerate}
The first part is just the simple observation that if $y\in c_b(x)$, $y\notin c_{b'}(x)$ then $\lvert c_b(x)\cap c_{b'}(x)\rvert \leq k-1$. This yields the bound 
\begin{align*}
    \lvert A_b\cap A_{b'}\rvert\leq \prod_{i=1}^n \lvert c_b(x_i)\cap c_{b'}(x_i)\rvert\leq k^n(\frac{k-1}{k})^{d_H(b,b')},
\end{align*}
where $d_H(b,b')=\lvert\{i\in[n]\;:\; b_i\neq b_i'\}\rvert$ is the Hamming distance between $b$ and $b'$. 

This upper bound on the sizes of the pairwise intersections naturally suggests us to employ the inclusion-exclusion principle to lower bound the size of the entire union. 
Unfortunately, directly applying it to the entire set $\{A_b\}_{b\in \{0,1\}^n}$ does not work, because the average hamming distance is too small. To overcome this issue we restrict ourselves to a random subset of $\{0,1\}^n$. Such a set is both large and has a large average Hamming distance.

\section{Sample Compression Proofs}\label{sec:Sample Compression Proofs}
Here we prove our main results concerning sample compression.
We begin with proving equivalence between variable size compressibility and learnability in Section \ref{sec:Learnability is Equivalent to Variable Size Compressibility}.
Then in Section \ref{sec:Impossibility Results for List Sample Compression} we prove our main results on sample compression: \Cref{thm:2-comp,t:2vsk,t:1vsk}.

\subsection{Learnability is Equivalent to Variable Size Compressibility}\label{sec:Learnability is Equivalent to Variable Size Compressibility}

The next definition extends the concept of sample compression schemes to permit the compression size to vary based on the size of the input sample.
\begin{definition}[Variable Size Compression]
%Let $k,k'\in\mathbb{N}$. 
A $k'$-list concept class $\C$ is $k$-list (variable-size) compressible
if there exist a sublinear sequence $d(n)=o(n)$ and a reconstruction function $\rho:(\X\times\Y)^\star\to {\Y \choose k}^\X$ such that the following holds.
For every sample $S\in(\X\times\Y)^n$ there exists $S'=\big((x_1,y_1),\ldots (x_{d(n)},y_{d(n)})\big)$, where $(x_i,y_i)\in S$ for all $i\leq d(n)$ such that
    \begin{equation}\label{eq:compressibilityvar}
    \mathtt{L}_S(\rho(S')) \leq \inf_{c\in \C}\mathtt{L}_S(c).
    \end{equation}
    If \Cref{eq:compressibilityvar} only holds for $\C$-realizable samples then we say that $\C$ is $k$-list (variable-size) compressible in the realizable case.
    As an important special case, we say that $\C$ is compressible of logarithmic size if $d(n)=O(\log n)$. 

% if there exists a $k$-list learning rule $\rho$ and a sequence $d(n) = o(n)$ such that the following holds. 
% For every sample $S\in(\X\times\Y)^n$ there exists $S'=\big((x_1,y_1),(x_2,y_2),\dots (x_{d(n)},y_{d(n)})\big)$, where $(x_i,y_i)\in S$ for all $i\leq d(n)$ such that
%     \begin{equation}\label{eq:compressibility}
%     \mathtt{L}_S(\rho(S')) \leq \inf_{c\in \C}\mathtt{L}_S(c).
%     \end{equation}
%     If \Cref{eq:compressibility} only holds for $\C$-realizable samples then we say that $\C$ is $k$-list compressible in the realizable case. \textcolor{red}{Shay: I think we should define compressibility only in the finite case. That is, not to include sublinear compression in the definition. This is the focus of the paper and there is only one secondary result about sublinear compression schemes. Maybe we can give the main definition for the finite case and as a variant introduce the variant of variable size sample compression schemes.}
% Let $\C$ denote the image of $\rho$:
% \[\H = \bigl\{\rho(S) : S\in(\X\times\Y)^d\bigr\}.\]
% If $\H\subseteq\C$ then we say that $\C$ is properly compressible.  
\end{definition}

We state and prove two basic results regarding the connection between compressibility and learnability: (i) an equivalence between learnability and (variable-size) compressibility,
and (ii) a quantitative bound on the gnerealization error of sample compression schemes.

\begin{theorem}\label{k-Learning vs O(log(n))- Compressing}[Learnability and Logarithmic-Compressibility]
    % A  concept class $\F\subset \Y^\X$ is $k$-list learnable iff it is $k$-list compressible with logarithmic variable size of $f(n)=O(\log n)$.
    Let $k,k'\in\mathbb{N}$. Then, the following statements are equivalent for a $k'$-list concept class $\C$:
    \begin{enumerate}
        \item $\C$ is $k$-list learnable in the agnostic setting.
        \item $\C$ is $k$-list learnable in the realizable setting.
        \item $\C$ is $k$-list compressible in the agnostic setting with compression size $d(n)=O(\log n)$.
        \item $\C$ is $k$-list compressible in the realizable setting with compression size $d(n)=O(\log n)$.
    \end{enumerate}
\end{theorem}
% Given \Cref{k-Learning vs O(log(n))- Compressing} we relinquish the distinction between the realizable and agnostic cases, and instead refer to concept classes as $k$-list compressible or $k-$list learnable when they are learnable or compressible in either of the equivalent definitions.
\begin{proposition}[Compression Size vs.\ Generalization Error]\label{prop:scsgen}
Let $k,k'\in\mathbb{N}$, $\C$ be a $k'$-list concept class that is compressible with compression size $d(n)$, and set $\A:=\rho\circ\kappa$. Then for any $\D$ realizable distribution and $\varepsilon>0$ we have for all $n>0$
\begin{align*}
\Pr\Big(\mathtt{L}_{\D}\big(\A(S)\big)>\varepsilon\Big)\leq 2\exp\big(d(n)\ln(n)-\varepsilon^2n\big). 
\end{align*}
And for all $n>0$ the learning curve $\varepsilon_n(\A\vert \D)$ satisfies
\begin{align*}
    \varepsilon_n(\A\vert \D)\leq \sqrt{\frac{(d(n)+1)\ln(n)}{n}}+\frac{2}{n}.
\end{align*}
\end{proposition}
The proofs of Theorem \ref{k-Learning vs O(log(n))- Compressing} and Proposition~\ref{prop:scsgen} are adaptations of the proofs of the classical cases of $k=1$, 
which can be found e.g.\ in \cite{david2016statistical}. 
For completeness, we repeat the argument with the necessary adjustments below.

 \begin{proof}[Proof of Theorem \ref{k-Learning vs O(log(n))- Compressing} and Proposition~\ref{prop:scsgen}] \label{proof of k-Learning vs O(log(n))- Compressing}
    We will show that $1\implies2\implies 4\implies3\implies1$,where the proof of $3\implies 1$ implies Proposition~\ref{prop:scsgen}.
    The direction $1\implies 2$ is clear by definition.  Next, we show that $4\implies 3$. Indeed let $(\rho,\kappa)$ be a realizable $k$-list compression scheme for $\C$. Let $S\subset(\X\times \Y)^*$ be some sample and take $S'\subset S$ to be a $\C$ realizable subsample of maximal size. by definition $\lvert S\rvert\cdot \mathtt{L}_S(c)\geq \lvert S\setminus S'\rvert$ for any $c\in \C$. Now  since $S'$ is realizable we have $\mathtt{L}_{S'}\big(\rho(\kappa(S'))\big)=0$ implying \begin{align*}
    \lvert S\rvert \cdot \mathtt{L}_S\Big(\rho\big(\kappa(S')\big)\Big)\leq \lvert S\rvert \cdot \mathtt{L}_{S'} \Big(\rho\big(\kappa(S')\big)\Big)+\lvert S\setminus S'\rvert =\lvert S\setminus S'\rvert \leq \lvert S\rvert \cdot\inf_{c\in\C}\mathtt{L}_S(c).
     \end{align*}
     Therefore $\rho$ is also an agnostic $k$-list reconstruction function for $\C$.

     \[\]
     Now we show that $3\implies 1$. Let $\rho$ be a reconstructor for $\C$ of size $\{d(n)\}_{n=1}^\infty$ for some sequence $d(n)=O(\log(n))$. We will show that $\A(S):=\rho\big(\kappa(S)\big)$ is  a $k$-list learner for $\C$. Fix some learning problem  $\D$, and draw a sample $S$ of size $n$ according to $\D$. define 
     \begin{align*}
         \mathcal{T}:=\{(x_1,y_1),(x_2,y_2),\dots (x_{d(n)},y_{d(n)})\;:\;\forall 1\leq i\leq d(n)\; (x_j,y_j)\in S\}.
     \end{align*}
     The set of all possible inputs to our reconstructor (with respect to the sample $S$).
     For any $T\in \mathcal{T}$, define \begin{align*}
         &S_T=S\setminus T=\{(x,y)\in S\;: (x,y)\notin T\},
         \\&h_T=\rho(T).
     \end{align*}
     For any $s=(x,y)\in S$  define the random variable $X_{s,T}=1_{y_\in h_T(x_)}$. So 
     \begin{align*}
         \mathtt{L}_{S_T}(h_T)=\frac{1}{d(n)}\sum_{s\in S_T}X_{s,T}.
     \end{align*}
     Now since clearly, $h_T$ is independent of $S_T$ and the $\{X_{s,T}\}_{s\in S}$ are independent of each other we may use Hoeffding’s inequality to get \begin{align*}
         \Pr(\lvert \mathtt{L}_{S_T}(h_T)-\mathtt{L}_\D(h_T)\rvert >\varepsilon)\leq 2e^{-\varepsilon^2n},
     \end{align*}
     and by simple union bound deduce 
     \begin{align*}
         \Pr(\exists T \; \text{ s.t } \; \lvert \mathtt{L}_{S_T}(h_T)-\mathtt{L}_\D(h_T)\rvert >\varepsilon )\leq 2n^{d(n)} e^{-\varepsilon^2 n}=2\exp(d(n)\ln(n)-\varepsilon^2n).
     \end{align*}
    Finally, we note that on the above event, we have $\lvert \mathtt{L}_\D(\A)-\mathtt{L}_S(\A)\rvert \leq \varepsilon$. hence \begin{align*}
    \Pr\big(\lvert \mathtt{L}_\D(\A)-\mathtt{L}_S(\A)\rvert >\varepsilon\big)\leq 2\exp(d(n)\ln(n)-\varepsilon^2n).
    \end{align*}
From this, we can use $d(n)=O(\log(n))$ to deduce the desired result.

\raggedright
Finally, we show that $2\implies 4$.
Let $\A$ be a $k$-list  learning rule for $\C$. Let $\varepsilon>0$ be such that $\varepsilon<\frac{1}{2(k+1)}$and take $d$ large enough such that the learning curve $\varepsilon_d(\A\vert \D)\leq \varepsilon$ for any realizable $\D$. We will show that $\C$ has a compression scheme of size $d(n)=\frac{d\log(3n)}{\varepsilon^2}$, note that $\varepsilon,d$ are constants so this indeed gives a logarithmic compression size.

We define the reconstruction function $\rho$  as follows. Given a sample $S$ of size $d\cdot T$ partition it into $T$ samples of size $d$ so $S=(S_1,S_2,\dots, S_T)$, then we can think of $\rho(S)$ as a majority vote of the $\{A(S_t)\}_{t=1}^T$, formally for any $x\in \X$ we define $\phi_x:\Y\to \Bbb{N}$, $\Phi_x:{\Y\choose k}\to \Bbb{N}$ by \begin{align*}
    &\phi_x(y)=\lvert \{t\in[T]\;:\; y\in A(S_t)(x)\}\rvert ,
\\&\Phi_x(Y)=\sum_{y\in Y}\phi_x(y).
\end{align*}
and then define $\rho(S)(x)$ by \begin{align*}
    \rho(S)(x)=\underset{Y\in{\Y \choose k}}{\text{argmax }}\Phi_x(Y)
\end{align*}
where ties are resolved arbitrarily.
Now we need to show that $\rho$ is indeed a reconstruction function for $\F$.
Fix some realizable sample $S=\{(x_i,y_i)\}_{i=1}^n$ of size $n$
  and define \begin{align*}
    \H:= \Big\{\A(S')\;:\; S'\subset {S \choose d}\Big\}.
\end{align*}
Note that for any realizable distribution $\D$ over $S$  there is some $h\in \H$ such that $\mathtt{L}_{\D}(h)\leq \varepsilon$. Now let us consider the zero-sum game between a learner and an adversary that goes as follows, the learner picks $h\in \H$ and the adversary picks $(x_i,y_i)\in S$, the learner wins if and only if $y_i\in h(x_i)$. With this view, the above remark states that for any randomized strategy of the adversary, the learner has a deterministic strategy that loses with probability at most $\varepsilon$. Hence by the well-known min-max theorem, there is some randomized strategy,  for which the expected loss of the learner is at most $\varepsilon$ for any randomized strategy of the adversary. Let $\mu$ be the distribution over $\H$ that induce the above strategy and let $\{H_t\}_{t=1}^T$ be $T$ elements of $\H$ independently drawn according to $\mu$ where $T=\frac{\log(3n)}{\varepsilon^2}$. We also define \begin{align*}
    &X_{t,i}=1_{y_i\in H_t(x_i)},
    \\&X_i=\frac{1}{T}\sum_{t=1}^T X_{t,i}.
\end{align*}
note that $EX_{t,i}=EX_i\geq 1-\varepsilon$ by the choice of $\mu$.
Now by Hoeffding's inequality, we have \begin{align*}
    &\Pr(\lvert X_i-EX_i\rvert >\varepsilon)\leq 2\exp (-\varepsilon^2T)=2e^{\frac{-\varepsilon^2 \log(3n)}{\varepsilon^2}}<\frac{1}{n},
    \\&\Pr\Big(\bigcup_{i=1}^n\lvert X_i-EX_i\rvert >\varepsilon\Big)<1.
\end{align*}
Hence, there are some $\{h_t\}_{t=1}^T$ in $\H$ such that \begin{align*}
    \frac{1}{T}\lvert \{t\in [T]\;:\; y_i\in h_t(x_i)\}\rvert>EX_1-\varepsilon>1-2\varepsilon>\frac{k}{k+1}.
\end{align*}
     Note that $h_t=A(S_t)$ for some $S_t\subset S$, $\lvert S_t\rvert =d$, we claim that $\kappa(S)=(S_1,S_2,\dots S_T)$ will give the desired result. And indeed given some $(x_i,y_i)\in S$ and using the same $\phi_{x_i}$, $\Phi_{x_i}$ as above we see that $y_i\notin \rho(\kappa(S))$ if and only if there at least $k$ other elements in $\Y$ at which $\phi_{x_i}$ is more then $\phi_{x_i}(y_i)$. But by the above $\phi_{x_i}(y_i)>\frac{kT}{k+1}$, hence we have $k+1$ points at which $\phi_{x_i}$ is at more then $\frac{kT}{k+1}$ implying that 
     \begin{align*}
         kT=\sum_{t=1}^T \lvert h_t(x_i)\rvert =\sum_{y\in \Y}\phi_{x_i}(y)>(k+1)\frac{kT}{k+1}=kT.
     \end{align*}
     Which is clearly impossible.
     
 \end{proof}

\subsection{Impossibility Results for List Sample Compression}\label{sec:Impossibility Results for List Sample Compression}

In this section, we prove Theorems \ref{thm:2-comp}, \ref{t:2vsk} and \ref{t:1vsk}. We start by giving technical background and stating some lemmas that are required for our proofs in Section~\ref{sec:tech}. Then, in Section~\ref{sec:mainlemma} we prove Lemma \ref{partial learnable class with large disambiguation function} which is key to the proofs of both  \Cref{t:2vsk,t:1vsk}; this lemma uses a direct sum argument to construct a $1$-list learnable class that is not $k$-list coverable.  We then utilize this lemma to prove Theorem \ref{t:2vsk} and Theorem $\ref{thm:2-comp}$ in Section~\ref{sec:2-comp} and Theorem \ref{t:1vsk} in Section~\ref{sec:proof1vsk}.

\subsubsection{Technical Background}\label{sec:tech}
Below we introduce definitions and claims that will be used throughout our proofs in this section.

\paragraph{Sauer-Shelah-Perles Lemma for List Learning.}%\label{sec: Sauer lemma}
We use the following version of the Sauer-Shelah-Perles (SSP) Lemma by \cite{Charikar2022ACO} 
\begin{lemma}[SSP for DS Dimension~\citep{Charikar2022ACO}]\label{Sauer–Shelah lemma}
    Let $\F\subset \Y^n$ be a function class with $d=\DS_k(\F)$, $m=\lvert \Y\rvert$, then we have \begin{align*}
        \lvert \F\rvert \leq k^{n-d}\sum_{i=0}^d {n\choose i}{m\choose k+1}^i\leq k^{n}n^{d}m^{(k+1)d}.
    \end{align*}
\end{lemma}

\paragraph{Partial Concepts.}
We next introduce some useful results from the theory of partial concept classes~\citep{Alon2021ATO}.
\begin{definition}
    Given a domain $\X$, a partial concept $c$ is a function whose domain $\supp(c)$ is a subset of $\X$, if $\supp(c)=X$ then $c$ is referred to as a total concept. A
    partial concept class is a collection of partial functions.
\end{definition}
Note that in the above definition, we may take $c$ to be a $k$-list function. However, for our results, this will not be necessary, as all our partial functions will be of the standard $k=1$ type.

A standard way to model partial concept classes is to introduce a new label ``$\star$" for the inputs on which any function is not defined. In this view, a partial concept class is simply a concept class $\C\subset (\Y\cup \{\star\})^\X$, where for any $c\in\C$, we have $\supp(c)=\{x\in \X\;:\; c(x)\neq \star\}=c^{-1}(\Y)$.

The definitions of standard learning theory extend naturally into the partial concept class 

\begin{definition}
    A sample $S$ of size $n$ is realizable by a $k$-list partial function $c$ if for any $(x,y)\in S$ we have $x\in \supp(c)$, $y\in c(x)$. A learning problem $\D$ is realizable by a partial concept class $\C\subset ({\Y\choose k}\cup\{\star\})^\X$ if for any $n\in \mathbb{N}$ a random sample $S$ of size $n$  drawn independently according to $\D$ is realizable by some $c\in \C$ with probability 1.
\end{definition}
Note that a realizable sequence can not contain the $\star$ symbol in it.
Now the notions of compression and learnability can be defined for partial concept classes in the same way as they are defined for total $k$-list concept classes.
\begin{definition}
  A partial concept class $\C\subset  (\Y\cup \{\star\})^\X$ is $k$-list covered  by a $k$-list class $\H\subset {\Y \choose k}^\X$ (and respectively $\H$ is a $k$-list covering of $\C$ ) if for any $c\in \C$, there exists $h\in \H$ such that $c(x)\in h(x)$ for all $x\in \supp(c)$. In such cases, we write $\C\dis\H$ and $c\dis h$
\end{definition}
Note that in the standard case of $k=1$, covering is equivalent to disambiguation. Although for $k>1$ covering can be seen as a generalization of disambiguation, it can also be studied in the context of total $k$-list classes.

\begin{definition}
    For any $\F\subset (\Y\cup \{\star\})^\X$, $k\in \Bbb{N}$ we define its $k$-list covering size
    \begin{align*}
        &\mathtt{C}_{\F}(n,k):=\sup \{\mathtt{C}(\F\vert_S,k)\;:\; S\subset\X,\; \lvert S\rvert =n\},
    \end{align*} 
    where $C(\F\vert _S,k)$ is the size of the minimal  $k$-list cover of $\F\vert_S$, i.e.\
    \begin{align*}
       \mathtt{C}(\F\vert_S,k):= \inf \bigg\{\lvert \H\rvert \;:\; \H\subset {\Y\choose k}^\X,\; \F\vert_S\dis\H\bigg\}.
    \end{align*}    
\end{definition}

\paragraph{Direct Sum.}%\label{Direct Sum of Learning Tasks}
We extend the definition of cartesian product to partial concepts.
\begin{definition}
    Given a $k$-list partial function $f:\U\to {\Y\choose k}\cup\{\star\}$ and a $k'$-list partial function $g:\V\to {\Z\choose k'}\cup\{\star\}$, 
    define their (cartesian) product $f\otimes g: \U\otimes \V\to {\Y\otimes \Z\choose kk'}\cup\{\star\}$ as follows
    \begin{align*}
        (f\otimes g)(u,v)=
        \begin{cases}
        \big\{(y,z)\;:\; y\in f(u), z\in g(v)\big\} &u\in\supp(f), v\in\supp(g),\\      
        \star                                       &\text{else.}
        \end{cases}
    \end{align*}
    Similarly, given a $k$-list (partial) concept class $\F\subset {\Y\choose k}^\U$ and a $k'$-list (partial) concept class $\G\subset {\Z\choose k'}^\V$, 
    define their product $\F\otimes\G$ as
    \begin{align*}
        \F\otimes\G=\{f\otimes g\;:\; f\in \F,\;g\in \G\}.
    \end{align*}
\end{definition}
% We extend the above definition to partial functions by defining $\supp(f\otimes g)=\supp(f)\otimes \supp(g)$, meaning that $f\otimes g$ is defined on $(u,v)$ if and only if $f$ is defined on $u$ and $g$ is defined on $v$. 
We primarily focus on the case where $k=k'=1$, simplifying the definition to
\begin{align*}
(f\otimes g)(u,v)=\big(f(u),g(v)\big).
\end{align*}

\begin{definition}[Product of Learning Rules]
Given a $k$-list learning rule $\A_1:(\X_1\times \Y_1)^\star \to {\Y_1 \choose k}^{\X_1}$ and a $k'$-list learning rule $\A_2:(\X_2\times \Y_2)^\star \to {\Y_2 \choose k'}^{\X_2}$.
We define their product $\A_1\otimes \A_2:((\X_1\otimes\X_2) \times(\Y_1\otimes \Y_2))^\star \to {\Y_1\times\Y_2 \choose k\cdot k'}^{\X_1\times\X_2}$ as follows
\[
(\A_1\otimes \A_2)(S)= \A_1(\pi_1(S))\otimes\A_2(\pi_2(S)),
\]
where $\pi_b\bigl(\{(x_{i,1}, x_{i,2}), (y_{i,1}, y_{i,2}) \}_{i=1}^n\bigr) = \{(x_{i,b}, y_{i,b})\}_{i=1}^n$.
\end{definition}

\begin{remark}[Products of Learning Rules and Sample Compressions Schemes]
Notice that if $\F$ is $k$-list compressible with comprssion size $d_1$ and $\G$ is $k'$-list compressible with compression size $d_2$ then $\F\otimes \G$ is $kk'$-list compressible with compression size $d_1+d_2$. Moreover we have \begin{align*}
\rho\big(\kappa(S)\big)=\rho_1\big(\kappa_1(\pi_1(S))\big)\otimes\rho_2\big(\kappa_2(\pi_2(S))\big).
\end{align*} 

Similarly if $\F$ is $k$-list learnable and $\G$ is $k'$-list learnable then $\F\otimes\G$ is $kk'$-list learnable.    
Moreover, if $\A_1$ is $k$-list learner for $\F$ and $\A_2$ is $k'$ list learner for $\G$  then $\A_1\otimes\A_2$ is $kk'$ list learner for $\F\otimes\G$ 
satisfying 
\begin{align*}
    \varepsilon(n\vert \A_1\otimes \A_2)\leq \varepsilon(n\vert A_1)+\varepsilon(n\vert A_2).
\end{align*}
\end{remark}

\subsubsection{Main Lemma: Coverability Under Direct-Sum}\label{sec:mainlemma}
In this part, we prove a main lemma on direct sum of partial concept classes and use it to deduce Lemma \ref{partial learnable class with large disambiguation function}, establishing the existence of a learnable partial class that is not $k$-list coverable, which is a key part in the proof of \Cref{constant compression 1 learnability} and \Cref{constant compression 2 learnability finite label space}.

\begin{lemma}[Main Direct Sum Lemma]\label{prodCover} 
    Let $\F\subset \Y^{\U}$, $\G\subset \Z^{\V}$ be  partial concept classes, then for any $k,k',n,n'\in \Bbb{N}$ we have 
    \begin{align*}
         \min \Big(\mathtt{C}_\F(n,k),{\mathtt{C}_\G(n',k')} \Big)\leq \min(n, n')\cdot \mathtt{C}_{\F\otimes \G}(n\cdot n',k+k').
    \end{align*}
    % \begin{align*}
    %      \max\bigg(\min \Big(\mathtt{C}_\F(n,k),\frac{\mathtt{C}_\G(n',k')}{n} \Big),\min \Big(\frac{\mathtt{C}_\F(n,k)}{n'},{\mathtt{C}_\G(n',k')} \Big)\bigg)\leq \mathtt{C}_{\F\otimes \G}(n\cdot n',k+k').
    % \end{align*}
    In particular, if $\F$ is not $k$-list coverable and $\G$ is not $k'$-list coverable then $\F\otimes \G$ is not $(k+k')$-list coverable.
\end{lemma}  
The proof of Lemma \ref{prodCover} requires the following technical claim

\begin{claim}\label{c:technical}
    Assume $h:(\U\times \V)\to {\Y\otimes \Z \choose k+k'}$ is a $(k+k')$-list function.
    Then for every $(u,v)\in \U\times \V$ at least one of the following holds:
    \begin{itemize}
        \item $\pi_2(h(u,v)):= \{z\in\Z : (y,z)\in h(u,v) \text{ for some}  y\in\Y\}$ has at most $k'$ distinct labels.
        \item For all $z\in \Z$ the set $\{y\in \Y: (y,z)\in h(u,v)\}$ has at most $k$ distinct labels.
    \end{itemize}
\end{claim}
\begin{proof}[Proof of claim \ref{c:technical}]
 Let $z'\in Z$, a simple computation gives 
    \begin{align*}
        k+k'=|h(u,v)|&=\sum_{z\in\pi_2(h(u,v))}|\{y\in \Y: (y,z)\in h(u,v)\}|
        \\&\leq |\{y\in \Y: (y,z')\in h(u,v)\}| +\sum_{\substack{ z\in\pi_2(h(u,v))\\z\neq z'}}1
        \\&=|\{y\in \Y: (y,z')\in h(u,v)\}|+|\pi_2(h(u,v))|-1.
    \end{align*}
    From which the claim follows immediately. 
\end{proof}
\begin{proof}[Proof of Lemma \ref{prodCover}]
% First, note that by symmetry it is enough to prove 
% \begin{align*}
% \min \Big(\mathtt{C}_\F(n,k),\frac{\mathtt{C}_\G(n',k')}{n} \Big)\leq \mathtt{C}_{\F\otimes \G}(n\cdot n',k+k')
% \end{align*}
Let $U\subset \U$, $V\subset \V$ samples and denote their sizes $|U|=n$, $|V|=n'$. Define $F=\F\vert_U$, $G=\G\vert_V$ and note that  $F\otimes G=(\F\otimes\G)\vert_{U\otimes V}$, and $|U\otimes V|=n\cdot n'$. Hence there is some  $H$ a $(k+k')$ list cover of $F\otimes G$ such that $|H|\leq \mathtt{C}_{\F\otimes\G}(n\cdot n',k+k')$. Without loss of generality we assume $n\leq n'$, so we need to show that either there exists a $k$-list cover of $F$ of size at most $n\cdot |H|$ or there exists a $k'$-list cover of $G$ of size at most $n\cdot \lvert H \rvert$.

The proof is based on a win-win argument; we show that if there exists a `good' $g\in G$ (to be defined below) then we we can cover $F$ using at most $\lvert H\rvert$ $k$-list functions.
And, otherwise, if no $g\in G$ is `good' then we can cover $G$ using at most $n\cdot\lvert H\rvert$ $k'$-list functions.

For any $h\in H$, $g\in G$, $v\in \supp(g)$, and $u\in U$ define $h_{g,v}(u)$ by 
\begin{align*}
    h_{g,v}(u)=\big\{y\in\Y \;:\; \big(y,g(v)\big)\in h(u,v)\big\}
\end{align*} 
Now Define $g$ as good if it satisfies \begin{align*}
    &(\forall f\in F)(\forall h\in H \text{ such that } f\otimes g\dis h)(\forall u\in \supp(f))(\exists v\in \supp(g) ) \;:\; |h_{g,v}(u)|\leq k.
\end{align*}
Now assume that there is some $g\in G$ that is good, we show how to use it to construct the desired $k$-list cover of $F$. Note that if $f\otimes g\dis h$ then $f\dis h_{g,v}$ for all $v\in \supp(g)$. Indeed for any $u\in \supp(f)$ we have $\big(f(u),g(v)\big)\in h(u,v)$ hence by definition $f(u)\in h_{g,v}$. Now in general the best bound we have on the size of $|h_{g,v}(u)|$ is $(k+k')$, so $h_{g,v}$ may not be a $k$-list function. To fix this, we wish to `trim' it somehow, removing some of the unnecessary labels. Note that in the above argument, we got $f(u)\in h_{g,v}(u)$ for all $v\in \supp(g)$, hence it seems natural to take intersection over all such $v$. For any $h\in H$ we define 
\begin{align*}
    h_g(u)=\bigcap_{v\in \supp(g)}h_{g,v}(u).
\end{align*} 
We already saw that if $f\otimes g \dis h$ we have that $f(u)\in h_{g,v}$ for all $v\in \supp(g)$. As a result $f(u)$ will be in the intersection and $f\dis h_{g}$. Now since $g$ for any $f\in F$, $h\in H$ such that $f\otimes g\dis h$ we have that for all $u\in U$ there is some $v\in \supp(g)$ such that $|h_{g,v}(u)|\leq k$. In particular, $|h_{g}(u)|\leq k$. So we have that $h_g$ is a $k$-list function\footnote{It is technically possible that the size of $h_v(u)$ is strictly less than $k$, in which case we add some labels to it arbitrarily to make it a $k$-list function } for all $h\in H$ such that there is some $f\in F$ with $f\otimes g\in h$. Hence we may define \begin{align*}
    H_g=\{h_g\;:\; \exists f\in F \; f\otimes g\in h\}.
\end{align*}
and by the above get that $H_g$ is a $k$-list cover of $F$, and clearly $|H_g|\leq|H|$.

Now we assume that no $g\in G$ is good and show that this implies the existence of a $k'$-list cover of $G$ of size at most $n\cdot|H|$. Using De Morgan's laws if no $g\in G$ is good then we have 
\begin{align*}
    (\forall g\in G)(\exists f\in F)(\exists h\in H \text{ s.t\ } f\otimes g\dis h)(\exists u\in \supp(f))(\forall v\in \supp(g) ) \;:\; |h_{g,v}(u)|> k.
\end{align*}
Now by Claim $\ref{c:technical}$ if $|h_{g,v}(u)|> k$ then with $\pi_2(y,z)=z$ the projection map we have $|\pi_2\big(h(u,v)\big)|\leq k'$. So the above becomes
\begin{align*}
    (\forall g\in G)(\exists f\in F)(\exists h\in H \text{ s.t.\ } f\otimes g\dis h)(\exists u\in \supp(f))(\forall v\in \supp(g) ) : |\pi_2\big(h(u,v)\big)|\leq k'.
\end{align*}
Note that if $f\otimes g\dis h$ then for all $u\in \supp(f)$,$v\in \supp(g)$ we have $\big(f(u),g(v)\big)\in h(u,v)$ hence $g(v)\in \pi_2\big(h(u,v)\big)$. Now by the above for all $g\in G$ there are some $f\in F$, $h\in H$, $u\in \supp(f)$ such that $|\pi_2\big(h(u,v)\big)|\leq k$ for all $v\in \supp(g)$, but we can't say that $v\to \pi_2\big(h(u,v)\big)$ is a $k'$-list function since we cant bound its size on $v\notin\supp(g)$. Luckily, when we try to cover $g$ we don't care for the value of the function on $v\notin \supp(g)$, so we may trim it arbitrarily. For any $u\in U$, $h\in H$ let $h_u:V\to \{\Z \choose k'\}$ be any $k'$-list function such that $\pi_2\big(h(u,v)\big)\subset h_u(v)$ for any $v\in V$ such that $|\pi_2\big(h(u,v)\big)|\leq k'$. Now by the above for any $g\in G$ there is some $h\in H$, $u\in U$ such that $g(v)\in h_u(v)$ for all $v\in\supp(g)$ hence if we define \begin{align*}
    H_U=\{h_u\;:\; h\in H,\;u\in U\}
\end{align*}
we have that $G\dis H_U$ and clearly $|H_U|\leq n\cdot |H|$.

\end{proof}

\begin{remark}
    Note that in the second case of the above proof, we get a bound of $n\cdot |H|$ on the size of the cover of $G$, as opposed to the bound $|H|$ in the case of the cover of $F$. This $n$ term appears since our functions are partial functions, and when dealing with total classes we may improve that bound. Indeed, similarly to the first case, we may try to trim $h_u$ by taking intersections and define  
\begin{align*}
    \pi_h(v)=\bigcap_{u\in U}h_u(v).
\end{align*}
And while at first, it seems that $\{\pi_h\; :\; h\in H\}$ will give a $k'$-list cover for $G$, a closer look will reveal that we don't necessarily have that it covers $G$ since $f\otimes g\dis$ only implies that $g(v)\in h_u(v)$ for $u\in \supp(f)$. Hence when we take the intersection over all of $U$ we can no longer guarantee that 
 $g(v)\in \pi_h(v)$, unless $\supp(f)=U$, i.e.\ $f$ is a total function.
\end{remark}

\begin{lemma}
\label{partial learnable class with large disambiguation function}
    For any $k\geq 1$ there is a partial concept class $\F_k\subset (\{0,1\}^k\cup \{\star\})^{\mathbb{N}^k}$ such that $\F_k$ is $1$-list learnable but $n^{k-1}\mathtt{C}_{\F^k}(n^k,k)\geq n^{\big(\log(n))\big)^{1-o(1)}}$    
\end{lemma}  

\begin{proof}
      The proof is by induction on $k$. For the base case $k=1$ we can take $\F_1$ to be the partial class from \cite{Alon2021ATO}. It is a learnable binary partial concept class over the natural number, and for any disambiguation of it $\C$ and $n>0$ there is some $S\subset \Bbb{N}$ of size $n$ such that $|\C\vert_S|\geq n^{\big(\log(n))\big)^{1-o(1)}}$. Thus it satisfies the required properties.
      
      For the induction step we simply define $\F_{k+1}=\F_k\otimes\F_1$. Indeed $\F_{k+1}\subset (\{0,1\}^k\cup \{\star\})^{\Bbb{N}^k}$ is a $1$-list learnable and from the induction assumption and Lemma \ref{prodCover} it satisfies
\begin{align*}
    n^{\big(\log(n))\big)^{1-o(1)}}\leq \min\Big(\mathtt{C}_\F(n,1),n^{k-1}\mathtt{C}_{\F^k}(n^k,k)\Big)\leq n^k \mathtt{C}_{\F^{k+1}}(n^{k+1},k+1).
\end{align*}
\end{proof}

% \subsection{List-Covers and Disambiguations}
% In this section, we restate the definitions of the minimal disambiguation and the free disambiguation, and we prove some of their basic properties.

% First, we note that if a sample $S$ is realizable by a partial concept class $(\Y\cup\{\star\})^\X$ then it is realizable by any disambiguation of it. Hence if a partial concept class has a $k$-list learnable (compressible) disambiguation is itself is $k$-list learnable (compressible) with the same $k$-list learner (reconstruction function). Showing that the other direction is not true is one of the major results of this paper, as well as that of \cite{Alon2021ATO} in $k=1$ setting.
% Now when searching for disambiguation for some partial concept class $\F\subset (\Y\cup\{\star\})^\X$ it is sometimes convenient to add some new labels to our label space $\Y$. 

\subsubsection{Proof of Theorems~\ref{thm:2-comp} and~\ref{t:2vsk}}\label{sec:2-comp}
In this part, we give proof for Theorems~\ref{thm:2-comp} and~\ref{t:2vsk}. We start by recalling the definitions of the minimal disambiguation and proving some of its basic properties and then use those properties with Lemma \ref{partial learnable class with large disambiguation function} to prove those Theorems.

\paragraph{Minimal Disambiguation.}
Recall the definition of the minimal disambiguation of a partial concept class.
\begin{definition*}[Definition~\ref{def:Minimal Disambiguation} restatement]
    Let $\C$ be a partial concept class and let $y_{\star}$ be a new label. For a partial concept $c$, let 
    $\bar c$ denote the completion of $c$ such that $\bar c(x)=y_\star$ whenever $c(x)=\star$.
    The class $\Bar\C = \{\bar{c} : c\in \C\}$ is called the \underline{minimal disambiguation} of $\C$.
\end{definition*}
Note that if $\C\subset (\Y\cup\{\star\})^\X$ is partial concept class that is $k$-list learnable then its minimal disambiguation $\Bar{\C}$ is $k+1$-list learnable. Indeed let $\A$ be a $k$-list learning rule for $\F$. For any sample $S=\{(x_i,y_i)\}_{i=1}^n$ define $S'=S\setminus(\X\times\{y_\star\})$, note that if $S$ is realizable by $\Bar{\C}$ then $S'$ is realizable by $\C$.
Now we may define  $\Bar{\A}$ a $(k+1)$-list learning rule for $\Bar{\C}$ by \begin{align*}
    \Bar{A}(S)(x)=A(S')(x)\cup \{y_\star\}
\end{align*}
And one can easily verify that $\Bar{A}$ is a learner for $\Bar{F}$.

Similarly if $\F$ is $k$-list compressible with reconstruction function $\rho$ of size $k_n$ we can easily verify that $\Bar{\F}$ is $k+1$-list compressible with reconstruction function $\Bar{\rho}(S)(x)=\rho(S')(x)\cup\{y_\star\}$. This phenomenon continues to hold in the following useful lemma concerning the covering size function
\begin{lemma}\label{growth rate of minimal disimbagution}
    Let $\F\subset (\Y\cup\{\star\})^\X$ be a learnable partial concept class over a finite label space $\lvert \Y\rvert<\infty$. Define $\lvert \Y\rvert=m$, $DS_1(\F)=d$, then we have \begin{align*}
        \mathtt{C}_\F(n,k)\leq (mn)^d\mathtt{C}_{\Bar{\F}}(n,k+1)
    \end{align*}
\end{lemma} 

\begin{proof}
    Let $S\subseteq \X$ be finite of size $n=\lvert S\rvert$, let $F=\F\vert_S$ and  $\bar F = \bar \F\vert_S$ and
    let $\H$ be a $(k+1)$-list covering for $\Bar{F}$.  We need to show that $F$ has a $k$-list cover of size $(mn)^d|H|$. $H$ is already a cover of $F$ but it consists of $(k+1)$-list functions, thus to construct the desired cover it is enough to replace each $h\in H$ with some 'sufficiently small' set of $k$-list functions that will cover it.
    Concretely, for any $h\in \H$ we define \begin{align*}
        &X_h=\{x\in {S} \;: \;y_\star\notin h(x) \}
        \\&F_h=\{f\vert_{X_h}\;:\; f(x)\in h(x) \text{ for all }x\in S\}.%=\F\vert_{X_h}\cap \Y^\X
    \end{align*}
We now introduce the $k$-list functions that will replace $h$ in the cover of $F$.
For each $y\in\Y$ let $A_y$ be some set of size $k$ such that $y\in A_y$, and for every $f\dis h$ define the $k$-list function $h_f$ by 
\begin{align*}
    h_f(x)=\case{A_{f(x)}}{x\in X_h,}{h(x)\setminus\{y_\star\}}{x\notin X_h,}
\end{align*}
Note that  $h_f$ is a $k$-list function: indeed for every $x\in X_h$ the set $A_{f(x)}$ has size $k$ and for $x\notin X_h$, we have that $y_\star\in h(x)$,
the size of $h(x)$ is $k+1$, and hence $h(x)\setminus\{y_\star\}$ has size $k$.
Now set
\begin{align*}
    H_F=\{h_f\;:\; h\in H, f\in F_h\}.
\end{align*}
We claim that $F\dis H_F$. Let $f\in F$. Then, since $H$ covers $\bar F$ and $\bar F$ disambiguates $F$ there must be some $h\in \H$ such that $f\dis h$. 
We show that $f\dis h_f$: let $x\in \supp(f)$; if $x\notin X_h$ then $f(x)\in h(x)\setminus\{y_\star\}=h_f(x)$. If $x\in X_h$ then $h_f(x)=A_{f(x)}$ and hence contains $f(x)$. To bound the size $H_F$ note that 
\begin{align*}
\lvert H_F\rvert &=\sum_{h\in H}\lvert\{f\vert_{X_h} : f\in F, f\dis h\} \rvert. 
\end{align*}
Now notice that $\{f\vert_{X_h} : f\in F, f\dis h\} \subseteq \{f\vert_{X_h} : f\in F, \supp(f)\subseteq X_h\} $, because $y_\star\notin h(x)$ whenever $x\in X_h$.
Thus, $\{f\vert_{X_h} : f\in F, f\dis h\}$ is a class of total functions with $k$-DS dimension at most $\DS_k(\F)=d$ and hence by the Sauer-Shelah-Perles Lemma:
\[
\lvert\{f\vert_{X_h} : f\in F, f\dis h\} \rvert \leq \lvert \Y\rvert^{d}\lvert X_h\rvert^d \leq (mn)^d,
\]
implying that 
 \begin{align*}
    \mathtt{C}_{\F}(k,n)\leq|\H_\G|\leq \sum_{h\in \H} |\G\vert_{h}|\leq (mn)^d|\H|.
\end{align*}
Taking minimum over all possible $\H$ with $\Bar{\F}\dis \H$ gives the desired claim.
\end{proof}

\begin{theorem*}[Theorem \ref{thm:2-comp} restatement]%\label{thm:2-comp}
There exists a concept class $\C$ over the label space $\Y=\{0,1,2\}$ such that:
\begin{itemize}
    \item $\C$ is $2$-list PAC learnable.
    \item $\C$ has no finite $2$-list sample compression scheme.
\end{itemize}    
\end{theorem*}

\begin{theorem*}[Theorem \ref{t:2vsk} restatement]
For any $k>0$ there exists a concept class~$\C_k$ over a \underline{finite} label space $\Y_k$ that satisfies the following:
    \begin{enumerate}
        \item $\C_k$ is $2$-list PAC learnable. 
        \item $\C_k$ has no finite $k$-list sample compression scheme.
    \end{enumerate}
\end{theorem*}

\begin{proof}[Proof of Theorems \ref{thm:2-comp} and \ref{t:2vsk}]
Let $\F_k$ be the partial concept class given by Lemma \ref{partial learnable class with large disambiguation function}, so $\F_k$ is learnable and $n^{k}\mathtt{C}_{\F_k}(n^k,k)\geq n^{\big(\log(n)\big)^{1-o(1)}}$. Now, set $\C_k=\Bar{\F}_k$ to be the minimal disambiguation of $\F_k$. We know that $\Bar{\F}_k$ is $2$-list learnable since $\F_k$ is $1$-list learnable. Now let $d=\DS_k(\F_k)$ and $m$ be the size of the label space of $\F_k$, then by Lemma \ref{growth rate of minimal disimbagution} we have   \begin{align*}
        (mn)^d\mathtt{C}_{\Bar{\F_k}}(n,k+1)\geq \mathtt{C}_{\F}(n,k)\geq n^{\big(\log(n)\big)^{1-o(1)}}.
    \end{align*}
for all $n>0$ large enough. Now since the label space of $\Bar{\F}_k$ is finite we may use the above with Lemma \ref{compression gives small growth rate}  to deduce that $\Bar{\F}_k$ is not $(k+1)$-list learnable.
We note that the case $k=1$ gives Theorem \ref{thm:2-comp}, where we can look at the proof of \ref{partial learnable class with large disambiguation function} to see that the label space of $\F_1$ will be of size $2$ hence the label space of $\C_1=\bar{\F_1}$ will be $3$.
\end{proof}

\subsubsection{Proof of Theorem~\ref{t:1vsk}}\label{sec:proof1vsk}
\paragraph{Free Disambiguation.} Recall the definition of the free disambiguation of a partial concept class.
\begin{definition*}[Definition~\ref{def:Free Disambiguation} restatement]   
    Let $\C$ be a partial concept class. For each $c\in \C$ let $y_c$ be a distinct new label. Let $\hat c$ denote the completion of $c$ such that $\hat c(x)=y_c$ whenever $c(x)=\star$. 
    The class $\Hat\C = \{\hat{c} : c\in \C\}$ is called the \underline{free disambiguation} of $\C$.
\end{definition*}
One can easily verify that for any partial concept class $\C\subset(\Y\cup\{\star\})^\X$ its free disimbaguation $\hat{\C}$ is $k$-list learnable (compressible) if and only if $\C$ is $k$-list learnable (compressible). Finally, we prove a short lemma that relates compressibility to polynomial growth of the covering size function.

\begin{lemma}\label{compression gives small growth rate}
    For any partial concept class $\F\subset (\Y\cup \{\star\})^\X$  that is $k$-list compressible with a compression scheme of size $d$  we have $\mathtt{C}_\F(n,k)\leq \lvert \Y\rvert^dn^{d}$ for all $n>0$.
\end{lemma}
\begin{proof}
Let $\F\subset \Y^\X$ be such class and without loss of generality assume that $|\X|=n$ (else we look on $\F\vert_S$ for some arbitrary $S\subset\X$ of size $n$).
Let $\rho$ be a $k$-list reconstructor of size $d$, so for any $f\in \F$ we have some $S\in (\X\times \Y)^d$ with $f\dis \rho(S)$. Hence we can define
 \begin{align*}
    \H=\{\rho\big(S\big)\;:\;S\in(\X\times \Y)^d\}
\end{align*}
And get that $\F\dis \H$, and clearly $|\H|=(|\X|\cdot\lvert \Y\rvert)^d=(n\lvert \Y\rvert)^d$.
\end{proof}

\begin{proof}[Proof of Theorem \ref{constant compression 1 learnability}]
     Let $\F_k$ be the partial concept class given by Lemma \ref{partial learnable class with large disambiguation function}, so $\F_k$ is learnable and $n^{k}\mathtt{C}_{\F_k}(n^k,k)\geq n^{\big(\log(n)\big)^{1-o(1)}}$. Hence by Lemma \ref{compression gives small growth rate} we have that $\F_k$ has no compression scheme of constant size $d$ for any $d>0$. Now $\C_k=\hat{\F}_k$ is the free disambiguation of $\F_k$, and since ${\F}_k$ is learnable but not $k$-list compressible so is $\C_k$. 
\end{proof}

% \section{Technical Background and Basic Results}\label{Preliminaries}
% \subsection{Sample Compression Schemes}

\section{Uniform Convergence Proofs}\label{sec:Uniform Convergence Proofs}
In this section we prove \Cref{t:UC} and its quantitative parallel \Cref{graph dimension vs DS dimension}. 

We start by proving Lemma \ref{graph dimension vs size of coverd functions} which is key to the proof. This lemma relates the graph dimension of a class to the number of sequences realizable by it. Then the theorem will follow by an application of the Sauer–Shelah-Perles lemma for list functions. 
\begin{lemma}\label{graph dimension vs size of coverd functions}
    Given a $k$-list class $\C\subset {\Y\choose k}^\X$ that $\Gk$ shatters  $S\in \X^n$, we let $p\in \Y^n$ be a pivot and let {$\{c_b\}_{b\in \{0,1\}^n}\subset \C\vert_S$} be witnesses for the shattering, such that $p_i\in c_b(x_i)$ if and only if $b_i=1$. Denote by $A_b=\{y\in \Y^n\;:\; \forall i\;\;y_i\in c_b(x_i)\}$, the collection of $c_b$ realizable functions. Then we have \begin{align*}
        \big|\bigcup_{b\in \{0,1\}^n}A_b\big|\geq \frac{(2k)^n}{4(2k-1)^n}k^n.
    \end{align*} 
\end{lemma}
\begin{proof}
    Clearly if $y\in A_b\cap A_{b'}$ then $y_i\in c_b(x_i)\cap c_{b'}(x_i)$ hence 
    \begin{align*}
        |A_b\cap A_{b'}|\leq \prod_{i=1}^n |c_b(x_i)\cap c_{b'}(x_i)|.
    \end{align*}
    Now in general we can only know that $|c_b(x_i)\cap c_{b'}(x_i)|\leq k$ but if $b_i\neq b_i'$ then $p_i$ will be in exactly one of $c_b(x_i),c_{b'}$. Hence in that case we have $|c_b(x_i)\cap c_{b'}|\leq k-1$. Now if we denote $d_H$ the Hamming distance $d_H(b,b')=|\{i\in [n]\;:\;b_i\neq b_i' \}|$ we can use the above to deduce that \begin{align*}
        |A_b\cap A_{b'}|\leq\prod_{b_i=b_i'}k\cdot \prod_{b_i\neq b_i'}(k-1)=k^{n-d_H(b,b')}(k-1)^{d_H(b,b')}=k^n(\frac{k-1}{k})^{d_H(b,b')}.
    \end{align*}
    Now we note that for any $I\subset \{0,1\}^n$ we have by the inclusion-exclusion principle \begin{align*}
        \big|\bigcup_{b\in \{0,1\}^n}A_b\big|\geq\big|\bigcup_{b\in I}A_b\big|\geq \sum_{b\in I}|A_b|-\sum_{b,b'\in I}|A_b\cap A_{b'}|\geq k^n|I|-\sum_{b,b'\in I }k^n(\frac{k-1}{k})^{d_H(b,b')}.
    \end{align*}

So to bound our desired expression we just need to find a large $I\subset \{0,1\}^n$ with a large average hamming distance, this is a standard problem in coding theory, and in our case, a probabilistic approach will suffice. Fix $m>0$ a constant that will be chosen later and pick $m$ elements from $\{0,1\}^n$ independently and uniformly, denote them by $I=\{B_i\}_{i=1}^m$. Note that for any $B,B'\in I$ we have that $d_H(B,B')=\sum_{i=1}^n 1_{B_i\neq B_i'}$ is a sum of independent Bernoulli random variables, hence by simple computation we have
\begin{align*}
    \E\Big(\frac{k-1}{k}\Big)^{d_H(B,B')}=\prod_{i=1}^n\E\Big(\frac{k-1}{k}\Big)^{1_{B_i\neq B_i'}}=\prod_{i=1}^n\Big(\frac{1}{2}+\frac{k-1}{2k}\Big)=\frac{(2k-1)^n}{(2k)^n}.
\end{align*}   
Form which we get \begin{align*}
    \E\Big[\sum_{B,B'\in I }k^n(\frac{k-1}{k})^{d_H(B,B')}\Big]=k^n{m \choose 2}\frac{(2k-1)^n}{(2k)^n}\leq m^2k^n\frac{(2k-1)^n}{(2k)^n}.
\end{align*}
Now putting this back into the inequality we calculated above we can see that 
\begin{align*}
     \big|\bigcup_{b\in \{0,1\}^n}A_b\big|\geq mk^n-m^2k^n\frac{(2k-1)^n}{(2k)^n}
\end{align*}
Now a simple calculation will reveal that the maximal value is attained when $m=\frac{(2k)^n}{2(2k-1)^n}$ for which we get the desired inequality 
\begin{align*}
     \big|\bigcup_{b\in \{0,1\}^n}A_b\big|\geq \frac{(2k)^n}{4(2k-1)^n}k^n.
\end{align*}
   
\end{proof}

\begin{theorem}\label{graph dimension vs DS dimension}
    Let $\C$ be some $k$-list concept class. Let $d=\DS_k(\C)$, $g=\Gk(\C)$ and $m=\lvert \Y\rvert$ the size of the label space of $\C$. Then we have \begin{align*}
        4g^dm^{(k+1)d}\geq \big(\frac{2k}{2k-1}\big)^g.
    \end{align*}
    In particular if $g=\infty$ then so is $d$ as the above can be simplified to \begin{align*}
        d= \tilde{\Omega}\Big(\frac{g}{k^2\cdot\log( m)+k\log (g)}\Big).
    \end{align*} 
\end{theorem}
\begin{proof}[Proof of Theorem \ref{t:UC} and \ref{graph dimension vs DS dimension} ]
    Let $S\in \X^g$ be a sample of size $g$ that is $\Gk$ shattered by $\C$.
    Define $\F=\F(\C\vert_S)=\{f\in \Y^g\;:\; \exists h\in\C \; ,\; f\dis h\vert_S\}$. By Lemma \ref{graph dimension vs size of coverd functions} we have that $|\F|\geq \frac{(2k)^g}{4(2k-1)^g}k^g$, while by Lemma \ref{Sauer–Shelah lemma} we have that $|\F|\leq k^g(gm^{k+1})^{\DS_k(\F)}$. Now we note that $\DS_k(\C\vert_S)=\DS_k(\F)$ since a sample is $\C\vert_S$ realizable if and only if it is $\F$ realizable, from which we may deduce the desired result. Note that Theorems \ref{t:finite DS-k vs learnability } and \ref{t:finite graph dimension vs unifrom convergence} which characterize $k$-list learnability and uniform convergence by the finiteness of DS-k and graph dimension, respectively. Hence
    We may deduce Theorem \ref{t:UC} immediately from the above.
\end{proof}

\section{Direct Sum and Open Questions}\label{sec:Direct Sum and Open Questions}
Below we introduce some open questions and research directions that arise from our study of direct sum. We start by focusing on studying how learning and uniform convergence rate scale under direct sum. We then ask similar types of questions for other learning resources
%such as the minimal $k$ for which a class is $k$-list learnable 
as well as on different combinatorial parameters that arise in learning theory.
%of direct sums, noting an interesting duality between those two approaches.
\subsection{Direct Sum of Learning Rates}\label{sec:Direct Sum of Learning Rates}
% \textcolor{red}{Here we discuss questions like: direct sum of realizable learning rates, direct sum of agnostic learning rates, direct sum of uniform convergence rates.}
One of the most natural questions regarding direct sums of learning problems is the following question: given two learning tasks, can we learn both of them in a faster way than learning each individuality? Perhaps the simplest case is of multiple instances of the same learning task.  Let $\C$ be a concept class and recall that for $r\in \Bbb{N}$, the $r$'th power of $\C$ is denoted by $\C^r=\underbrace{\C\otimes \C \dots \otimes \C}_{r \text{ times }}$. 
\begin{center}
\emph{How does the learning rate of $\C^r$ scale in terms of the learning rate of $\C$?}    
\end{center}
This problem can be investigated with respect to various formulations of `learning rate', for example:

% In our setting, the most natural definition of learning rate will be the learning curve, which leads us to Open question \ref{open:sumofcurves}
\begin{open*}[Open question \ref{open:sumofcurves} restatement ]
Let $\C\subseteq \Y^\X$ be a concept class, and let $\eps(n \vert \C)$ denote the realizable PAC learning curve of $\C$ (see Definition \ref{def:PAC}). 
For $r\in\mathbb{N}$ let $\C^r= \Pi_{i=1}^r \C$ be the $r$-fold cartesian power of $\C$.
By a union bound, learning each component independently gives
    \[\eps(n \vert \C^r) \leq  r\cdot \eps(n \vert \C).\]
Can the upper bound be asymptotically improved for some classes $\C$?
\end{open*}
Another natural version of the above is assuming a fixed marginal distribution $\D$:
\begin{definition}
   Let $\D$ be a fixed distribution over the domain $\X$ and let $\C$ be a concept class. 
   For any $c:\X\to \Y$ let $\D_c$ be the distribution in which $(x,y)~\sim \D_c$ satisfies $x\sim \D$ and $y=c(x)$. 
   Define the fixed-marginal learning curve $\varepsilon(n\vert \D,\C)$ by 
   \begin{align*}
 \varepsilon(n\vert \D,\C)=\inf_{\A}\sup_{c\in\C} \varepsilon_n(\D_c\vert \A)
\end{align*} 
where the infimum is taken over all learning rules $\A$.  
\end{definition}
Note that for any $c\in \C$ we have that $\D_c$ is a realizable distribution, hence $\varepsilon(n\vert \D,\C)\leq \varepsilon(n\vert \C)$ always holds. 

 For any $r>0$, let $\D^r$ be the product measure over $\X^r$. 
\begin{open}
    Similarly to the case of the PAC learning curve, a simple union bound will give the upper bound 
\begin{align*}
    \varepsilon(n\vert \D^r, \C^r)\leq r\cdot \varepsilon(n\vert \D,\C).
\end{align*}
Can the upper bound be asymptotically improved for some concept classes $\C$ and marginal distributions $\D$?
\end{open}

%Now in the above questions, we could trivially bound from above the error rate of $\C^r$ by $r$ times the error rate of $\C$. While this kind of bound is expected when dealing with direct sums problems, it is not always as clear in some learning models. Indeed when dealing with regret-based models the learner tries to compete with some optimal function $opt$, and by similar union-bound arguments, we expect that both the learner and $opt$ will have a similar additive upper bound. But it is not clear what happens to their difference.
%
%For example, consider the uniform convergence rate of a $k$-list concept class $\C$. 

One can ask similar questions about agnostic learning curves and uniform convergence.
However, in these cases the baseline additive upper bound does not apply. 
The reason is because these curves concern relative quantities (indeed, the agnostic learning curve measures the excess loss and uniform convergence curve measures the maximum difference between the empirical and population losses).

For instance,  given a distribution $\D$ over the product space $(\X^2\times \Y^2)$ defined with marginal distributions $\D_1,\D_2$ over $(\X\times\Y)$ we have by the union bound that $\mathtt{L}_\D(h_1\otimes h_2)\leq \mathtt{L}_{\D_1}(h_1)+\mathtt{L}_{\D_2}(h_2)$. Similarly if  $S=\{\big((x_{i,1},x_{i,2}),(y_{i,1},y_{i,2})\big)\}_{i=1}^n$, letting $S_b=\{(x_{i,b},y_{i,b})\}_{i=1}^n$, we have $\mathtt{L}_{S}(h_1\otimes h_2)\leq \mathtt{L}_{S_1}(h_1)+\mathtt{L}_{S_2}(h_2)$. These bounds, however, do not allow us to bound the \emph{difference} $|\mathtt{L}_D(h_1\otimes h_2)-\mathtt{L}_S(h_1\otimes h_2)|$ as needed to bound the uniform convergence rate. 

\begin{open}
Let $\varepsilon_{\mathtt{UC}}(n\vert \C)$ be the uniform convergence curve of $\C$ (Definition~\ref{def:UC}). 
How does $\varepsilon_{\mathtt{UC}}(n\vert \C^r)$ scale as a function of $\varepsilon_{\mathtt{UC}}(n\vert \C)$ and $r$? 
\end{open}

A similar phenomenon happens in the case of agnostic learning:
define the agnostic learning curve of a concept class $\C$ by
\begin{align*}
    \varepsilon_{\mathtt{agn}}(n\vert \C)=\inf_{\A}\sup_{\D}(\mathtt{L}_{\D,n}(\A)- \mathtt{L}_{\D}(\C)),
\end{align*} 
where the infimum is taken over all learning rules $\A$, the supremum is taken over all distributions, and $\mathtt{L}_{\D,n}(\A)=\E_{S\sim \D^n}[\mathtt{L}_\D(\A(S))]$, and $\mathtt{L}_\D(\C)=\inf_{c\in\C}\mathtt{L}_\D(c)$.
Here again we do not have simple bounds to the agnostic learning curve of $\C^r$ in terms of the agnostic learning curve of $\C$. 
\begin{open}
Let $\varepsilon_{\mathtt{agn}}(n\vert \C)$ be the agnostic PAC learning curve of $\C$. 
How does $\varepsilon_{\mathtt{agn}}(n\vert \C^r)$ scale as a function of $\varepsilon_{\mathtt{agn}}(n\vert \C)$ and $r$?     
\end{open}

\subsection{Direct Sum of Learnability Parameters}
Another important resource in the context of list learning is the minimal list size $k$ for which a given class $\C$ is $k$-list PAC learnable. This raises the following questions:
\begin{open}\label{prob:k1-k2}
Let $\C_1,\C_2$ be concept classes and assume $k_1$ and $k_2$ are the minimal integers such that $\C_1$ is $k_1$-list PAC learnable and $\C_2$ is $k_2$-list PAC learnable.
What is the minimal integer $k$ such that $\C_1\otimes\C_2$ is $k$-list PAC learnable?
How does it scale as a function of $k_1$ and $k_2$.
\end{open}
It is not hard to see that $k\leq k_1\cdot k_2$ by just learning each component independently and taking all possible pairs of labels in the marginal lists. We also show that $k\geq (k_1-1)\cdot(k_2-1)$ (see Equation~\eqref{eqn:min-k-size} below). However, it remains open to determine tight bounds on $k$.

We raise the parallel question regarding compressibility:
\begin{open}\label{prob:compression}
Let $\C_1,\C_2$ be concept classes and assume $k_1$ and $k_2$ are the minimal integers such that $\C_1$ is $k_1$-list compressible and $\C_2$ is $k_2$-list compressible.
What is the minimal integer $k$ such that $\C_1\otimes\C_2$ is $k$-list PAC learnable?
How does it scale as a function of $k_1$ and $k_2$.

    % Given concept classes $\F,\G$ such that $\F$ is not $k$-list compressible and $\G$ is not $k'$-list compressible can we show that $\F\otimes \G$ is not $(k+k')$-list compressible? Can we show that it is not $k\cdot k'$-list compressible? 
\end{open}  
A natural way to approach questions such as the ones above and in Section~\ref{sec:Direct Sum of Learning Rates} is by analyzing combinatorial parameters that capture the corresponding resources. 
% For example, for Question~\ref{open:sumofcurves} it is natural to consider the Daniely-Shwartz dimension which charaterizes learnability in the realizable setting, and to investigate how large can $\DS_k(\F\otimes\G)$ be as a function of $\DS_{k_1}(\F),\DS_{k_2}(\G)$.
\begin{open}
Let $\F,\G$ be concept classes, and let $\mathtt{dim}(\cdot)$ be either the Graph dimension, the Natarajan dimension, the Littlestone dimension, or the Daniely-Shwartz dimension.
How does $\mathtt{dim}(\F\otimes\G)$ scale in terms of $\mathtt{dim}(\F)$ and $\mathtt{dim}(\G)$?
% \begin{enumerate}
%     \item How does the graph dimension of $\F\otimes \G$ scales in terms of the graph dimensions of $\F$ and $\G$?.
%     \item How does the Natarajan dimension of $\F\otimes \G$ scales in terms of the Natarajan dimensions of $\F$ and $\G$?.
%     \item How does the Littlestone dimension of $\F\otimes \G$ scales in terms of the Littlestone dimensions of $\F$ and $\G$?
% \end{enumerate}
\end{open}

We next provide some preliminary results.

\begin{proposition}
    Let $\mathtt{d_N}(\cdot)$ be the Natarajan dimension, and let $\F$ and $\G$ be concept classes. Then, 
    \begin{align*}
    \mathtt{d_N}(\F)+\mathtt{d_N}(\G)-2\leq\mathtt{d_N}(\F\otimes \G)\leq \mathtt{d_N}(\F)+\mathtt{d_N}(\G).    
    \end{align*}
\end{proposition}
\begin{proof}
    Let $U=\{u_i\}_{i=1}^n$, $V=\{v_j\}_{j=1}^m$ be sets that are Natarajan shattered by $\F$ and $\G$ respectively so $\F\vert_U$ and $\G\vert_V$ contains a Cartesian product of the form
    \begin{align*}
        &\prod_{i=1}^n \{f_i,\tilde{f_i}\}\subset \F\vert_U,
        \\& \prod_{j=1}^m \{g_j,\tilde{g}_j\}\subset\G\vert_V, 
    \end{align*}
    where $f_i\neq \tilde{f_i}$, $g_j\neq \tilde{g_j}$ for all $i,j$. Now define $S=\{(u_i,v_1)\}_{i=2}^n\cup\{(u_1,v_j)\}_{j=2}^m$ and note that $|S|=|U|+|V|-2$. We claim that $S$ is Natarajan shattered by $\F\otimes\G$. Indeed, 
    \begin{align*}
       \prod_{i=2}^n \{(f_i,g_1),(\Tilde{f_i},g_1)\}\times \prod_{j=2}^m \{(f_1,g_j),(f_1,\Tilde{g_j})\} \subseteq(\F\otimes\G)\vert_S.
    \end{align*}
    Taking supremum over all such $U,V$ we get that $\mathtt{d_N}(\F)+\mathtt{d_N}(\G)-2\leq\mathtt{d_N}(\F\otimes \G)$. Now for the other part of the inequality let $S=\{(u_i,v_i)\}_{i=1}^n$ be some set that is Natrajan shattered by $\F\otimes\G$, meaning that \begin{align*}
    \prod_{i=1}^n\{(f_i,g_i),(\Tilde{f_i},\Tilde{g_i})\} \subset\F\otimes\G\vert_S.
    \end{align*}
Where $(f_i,g_i)\neq(\Tilde{f_i},\Tilde{g_i})$ for all $i$, but we can have $f_i=\Tilde{f_i}$ or $g_i=\Tilde{g_i}$. Now we can define \begin{align*}
    U=\{u_i\;:\; f_i\neq \Tilde{f_i}\}.
    V=\{v_i\;:\; g_i\neq \Tilde{g_i}\}.
\end{align*}
Note that since $(f_i,g_i)\neq(\Tilde{f_i},\Tilde{g_i})$ we have that $|U|+|V|\geq n$ and we claim that $U$ and $V$ are Natarajan shattered by $\F$ and $\G$ respectively. And indeed clearly we have 
\begin{align*}
&\prod_{u_i\in U}\{(f_i,\Tilde{f_i}\}\subset \F\vert_U.
\\&\prod_{v_i\in V}\{g_i,\Tilde{g_i}\}\subset\G\vert_V.    
\end{align*} 
By definition of $U$, $V$ we have $f_i\neq \Tilde{f_i}$ for all $i$ such that $u_i\in U$ and $g_i\neq \Tilde{g_i}$ for all $i$ such that $v_i\in V$. Hence we have \begin{align*}
    |S|\leq |U|+|V|\leq \mathtt{d_N}(\F)+\mathtt{d_N}(\G).
\end{align*}
Taking supermum over all such $S$ will give the desired result.    
\end{proof}

\begin{proposition}
    Let $\mathtt{LS}(\cdot)$ be the Littlestone dimension, then for any $\F,\G$ concept classes we have \begin{align*}
        \mathtt{LS}(\F\otimes \G)=\mathtt{LS}(\F)+\mathtt{LS}(\G)
    \end{align*}
\end{proposition}
\begin{proof}
 We use the fact that the Littlestone dimension equals the optimal mistake bound in deterministic online learning~\cite{littlestone:86}.
 Note that if $\A_1$ is an online learner that makes at most $n_1$ mistakes on $\F$ and $\A_2$ is an online learner that makes at most $n_2$ mistakes on $\G$ then $\A_1\otimes \A_2$ is an online learner that makes at most $n_1+n_2$ mistakes on $\F\otimes\G$. Hence using the equivalence between Littlestone dimension and mistake bound we deduce $\mathtt{LS}(\F\otimes \G)\leq \mathtt{LS}(\F)+\mathtt{LS}(\G)$.

 Now for the other direction let $T_1=(V_1,E_1)$ be a binary tree that is Littlestone shattered by $\F$ and let $T_2=(V_2,E_3)$ be a binary tree that is  Littlestone shattered by $\G$. We will use $T_1,T_2$ to construct a binary tree $T$ that is shattered by $\F\otimes \G$, and its depth is the sum of the depths of $T_1$ and $T_2$. Our approach is to take a copy of $T_1$ and concatenate a copy of $T_2$ to each of its leaves, modifying the vertices as necessary. Formally, let $r_2$ be the root of $T_2$ and $L_1$ be the set of leaves of $T_1$ and define $T=(V,E)$ by 
    \begin{align*}
        &V=\{(v,r_2)\;:v\in V_1\}\cup \{(l,v)\;:v\in V_2,\;l\in L_1\}
        \\&E=\{(v,r_2)\to(u,r_2)\;: v\to u\in E_1\}\cup\{(l,v)\to(l,u)\;:\; l\in L_1,\; v\to u\in E_2\}
    \end{align*}
    So each path $\sigma$ in $T$ is of the form 
    \[(r_1,r_2)\to(v_1,r_2)\dots \to(v_n,r_2)\to(v_n,u_1)\dots \to(v_n,u_m),\]
    where $\sigma_1=r_1\to v_1,\dots \to v_n$ is a path in $T_1$ and $\sigma_2=r_2\to u_1\dots \to u_n$ is a path 
    in $T_2$. Hence if $f\in \F$, $g\in \G$ are the functions that realize the paths $\sigma_1,\sigma_2$ then $f\otimes g$ realize $\sigma$, implying that $T$ is Littelstone shattered by $\F\otimes\G$, and we can easily verify that $d(T)=d(T_1)+d(T_2)$ where $d(\cdot)$ denote the depth of a binary tree.
\end{proof}

% \begin{proposition}
%     Let $\F$ be a $k$-list function class and $\G$ be a $k'$-list function class. then $\mathtt{G}_{kk'}$
% \end{proposition}
% \begin{proof}
%     % Let $U=\{u_i\}_{i=1}^n$ and $V=\{v_i\}_{i=1}^m$ be sets that are $\Gk$-shattered and $\mathtt{G}_{k'}$ shattered by $\F$ and $\G$ respectively. Let $p$ be a pivot for $U$ and $p'$ be a pivot for $V$, so we have some $\{f_{b}\}_{b\in\{0,1\}^n}\subset \F\vert_U$, $\{g_b\}_{b\in \{0,1\}^n}\subset \G\vert_V$ such that $p_i\in f_b(u_i)$ iff $b_i=1$ and $p_i'\in g_b(v_i)$ iff $b_i=1$. Now define $S=\{(u_i,v_1)\}_{i=1}^n\cup \{(v_i,u_1)\}_{i=1}^m$ and the pivot $q$ by \begin{align*}
%     %     q_i=\case{(p_i,p'_1)}{1\leq i\leq n}{(p_1,p_{i+1-n})}{n+1\leq i\leq n+m-1}
%     % \end{align*}.
%     % Note that for any $1\leq i\leq n$ we have $q_i=(p_i,p'_1)\in (f_b(u_i),g_{b'}(v_1))$ iff $b_i=1$ and $b_1'=1$
%     To study the graph dimension of $\F^k$ we reduce the problem into a VC dimension problem. For any $p\in \Y^\X$, $c\in\C$ we define $\chi_p(c)=\{x\in \X\;:\; p(x)\in c(x)\}$ and $\chi_p(\C)=\{\chi_p(c)\;:\; c\in \C\}$.
%     Note that 
%     \begin{align*}
%         \mathtt{G}_1(\C)=\sup_{p\in \Y^\X}\bigg(\mathtt{VC}(\chi_p(\C))\bigg),
%     \end{align*}
%     where $\mathtt{VC}$ is the standard VC dimension of a family of sets.  
% \end{proof}

\begin{proposition}\label{addtive dimension of direct sum}
    Let  $\F,\G$ be partial function classes, and let $k, k'\geq 1$. Then we have the following 
    \begin{enumerate}
        \item $\DS_{k\cdot k'}(\F\otimes \G)\geq \min \big(\DS_{k}(\F),\DS_{k'}(\G)\big)$.
        \item $\DS_{\min(k,k')}(\F\otimes\G)\geq \DS_k(\F)+\DS_{k'}(\G){-1}$.
    \end{enumerate}
\end{proposition}
\begin{proof}
    Let $U=\{u_i\}_{i=1}^n$ be a set that is $\DS_k$ shattered by $\F$ and let $V=\{v_i\}_{i=1}^m$ be a set that is $\DS_{k'}$ shattered by $\G$. Without loss of generality assume that $n\leq m$.
    Define $S=\{(u_i,v_i)\}_{i=1}^n$ and notice that if $f$ is a neighbor of $f'$ in the $u_i$ direction and $g$ is a neighbor of $g'$ in $v_i$ direction then $f\otimes g$ is a neighbor of $f'\otimes g'$ in the $(u_i,v_i)$ direction. Since each $f\in \F\vert_U$ has $k$ neighbors in each direction, and each $g\in \G\vert_V$ has $k'$ neighbors in each direction we have that each $f\otimes g\in \F\vert_U \otimes \G\vert_S$ has $kk'$ neighbors in each direction, implying the first claim. 
    
    For the second claim, 
    set $T=\{(u_1,v_i)\}_{i=2}^n\cup \{(u_i,v_1)\}_{i=2}^m\cup\{(u_1,v_1)\}$. In a similar way to the above, now $f\otimes g$ is a neighbor of $f'\otimes g$ in the $(u_i,v_1)$ direction and of $f\otimes g'$ in the $(u_1,v_i)$ direction. So each $f\otimes g\in \F\otimes \G\vert_T$ has $k$ neighbors in the $(u_i,v_1)$ direction and $k'$ neighbors in the $(u_1,v_j)$ direction for each $1\leq i\leq n$, $1\leq j\leq m$. From which the second claim follows.

    %we pick $u_0,v_0$ such that $u_0\notin U$, $v_0\notin V$ \textcolor{red}{Tom: I think cases where we there are no points outside the shattered set are not interesting but maybe we need to say something about this?} and 
    %set $T=\{(u_0,v_i)_{i=1}^n\}\cup \{(u_i,v_0)\}_{i=1}^m$. In a similar way to the above, we see that now $f\otimes g$ will be a neighbor of $f'\otimes g$ in the $(u_i,v_0)$ direction and of $f\otimes g'$ in the $(u_0,v_i)$ direction. so each $f\otimes g\in \F\otimes \G\vert_T$ has $k$ neighbors in the $(u_i,v_0)$ direction and $k'$ neighbors in the $(u_0,v_j)$ direction for each $1\leq i\leq n$, $1\leq j\leq m$. From which the second claim follows.
    
\end{proof}

Note that Proposition \ref{addtive dimension of direct sum} has direct implications relevant to Open Question~\ref{prob:k1-k2}.
Specifically, it implies that if $\F$ is not $k$-list learnable and $\G$ is not $k'$-list learnable then $\F\otimes\G$ is not $k\cdot k'$-list learnable. Conversely we know that if $\F$ is $k$-list learnable and $\G$ is $k'$ list learnable then $\F\otimes \G$ is $k\cdot k'$-list learnable. Thus, letting $K(\C)$ denote the minimal $k$ such that a concept class $\C$ is $k$-list learnable (or infinity if there is no such $k$) we can summarize the above as 
\begin{align}
\label{eqn:min-k-size}
(K(\F)-1)(K(\G)-1)\leq K(\F\otimes \G)\leq K(\F)\cdot K(\G) 
\end{align}
%Naturally one might want to improve the above bound and make it tight. 
Note that this result can be seen as parallel with Lemma \ref{prodCover} which implies that if $\F$ is not $k$-list coverable and $\G$ is not $k'$-list coverable then $\F\otimes \G$ is not $(k+k')$-list coverable. It will be interesting to sharpen the latter result, replacing the $(k+k')$ term by $k\cdot k'$ like in Equation~\ref{eqn:min-k-size} above. We may also ask similar questions about compressibility, an answer to which would be relevant to Open Question~\ref{prob:compression}.

Proposition~\ref{addtive dimension of direct sum} also has implications to Open Question~\ref{open:sumofcurves} for ($1$-list) PAC learnable classes.
Indeed, let $\C$ be a PAC learnable class. Then, by Item 2 in Proposition~\ref{addtive dimension of direct sum}, it follows that $\DS_1(\C^r)\geq r\DS_1(\C)$. Hence, since the Daniely-Shwartz dimension lower bounds the PAC learning curve~\citep{Charikar2022ACO} we get
\begin{align*}
\varepsilon(n\vert \C^r)\geq \frac{{\DS_1(\C^r)}}{n}\geq \frac{{r\cdot \DS_1(\C) {-r}}}{n}.
\end{align*}
Thus, if it turns out that the realizable PAC learning curve is $\Theta\bigl(\frac{\DS_1}{n}\bigr)$ then the above in combination with the naive union bound argument mentioned in Open Question~\ref{open:sumofcurves} would answer this question up to universal multiplicative constants.

\bibliography{references.bib}

\end{document}